\documentclass[11pt]{article}

\usepackage[left=1in,top=1in,right=1in,bottom=1in]{geometry}
\usepackage{amsmath,amssymb,amsthm,amsfonts}
\usepackage{times}
\usepackage{enumerate}
\usepackage{epsfig}
\usepackage{graphicx}
\usepackage{pgf}
\usepackage{chemarr}
\usepackage{float}

\usepackage[linesnumbered,ruled,commentsnumbered,noline]{algorithm2e}

\usepackage{cases}
\usepackage{caption}
\usepackage{subcaption}

\usepackage{url}
\usepackage{bm}
\usepackage[square, numbers]{natbib}
\usepackage{tcolorbox}
\include{psbox}
\usepackage[mathscr]{eucal}
\usepackage{color}
\usepackage{macro}

\newcommand{\salg}{\mathscr{F}}

\newcommand{\const}{\mathscr{C}}

\newcommand{\wtb}{\beta}

\newcommand{\bmn}{\mu}
\newcommand{\bcm}{\bm{C}}
\newcommand{\sdg}{\bm{D}}
\newcommand{\mat}{M}
\newcommand{\matlo}{A}
\newcommand{\matgl}{B}
\newcommand{\locm}{\Lambda}
\newcommand{\glom}{\Xi}
\newcommand{\loc}{\lambda^2} 
\newcommand{\glo}{\tau^2} 
\newcommand{\hyp}{\alpha}
\newcommand{\hpa}{\mfk{a}}
\newcommand{\hpb}{\mfk{b}}

\makeatletter
\def\blfootnote{\gdef\@thefnmark{}\@footnotetext}
\makeatother

\title{Infinite-dimensional optimization and Bayesian nonparametric learning of stochastic differential equations} 
\author{Arnab Ganguly$^{1}$\thanks{Research is supported in part by NSF DMS - 1855788 and Louisiana Board of Regents through the Board of  Regents Support \hs*{.4cm}  Fund (contract number: LEQSF(2016-19)-RD-A-04).}, \ Riten Mitra $^{2}$, \ Jinpu Zhou$^{1}$\footnotemark[1]\thanks{Research is supported in part by NSF DMS - 1855788.}}

\begin{document}
\blfootnote{Authors contributed equally.}
\maketitle
\footnotetext[1]{Department of Mathematics, Louisiana State University, aganguly@lsu.edu (AG), zjinpu1@lsu.edu (JZ). }
\footnotetext[2]{Department of Bioinformatics and Biostatistics, University of Louisville, 	ritendranath.mitra@louisville.edu.}

\begin{abstract}
The paper has two major themes. The first part of the paper establishes certain general results for infinite-dimensional optimization problems on Hilbert spaces. These results cover the  classical representer theorem and many of its variants as special cases and offer a wider scope of applications. The second part of the paper then develops a systematic approach for learning the drift function of a stochastic differential equation by integrating the results of the first part with Bayesian hierarchical framework. Importantly, our Baysian approach incorporates low-cost sparse learning through proper use of  shrinkage priors while allowing proper quantification of uncertainty through posterior distributions. Several examples at the end illustrate the accuracy of our learning scheme.\\

\noindent

\end{abstract}

\vspace{.1in}

\setcounter{equation}{0}
\renewcommand {\theequation}{\arabic{section}.\arabic{equation}}
\section{Introduction.}\label{intro}

The temporal dynamics of a variety of systems arising from systems biology, environmental science, engineering, physics, medicine can be captured by stochastic differential equations (SDEs) driven by appropriate drift function and noise (c.f \eqref{eq:sde0}). SDEs are also central to modern financial mathematics where they are used to model short term interest rates, asset and options pricing, their volatility. Understanding behaviors of these systems requires not just building mathematical models but integrating it with the available data. For instance, advanced technologies like single-cell imaging can attest to the stochasticity of cellular processes  \cite{ref:Friedman10, coulon2013}.  While this  molecular noise is a  rich source of information about the process dynamics, utilizing this source in a systematic manner  requires building  stochastic temporal models that are calibrated according to the available data. Building such data-driven  models characterizing the inner-workings of these systems is instrumental for advancement of quantitative biology and other quantitative disciplines.  
 
 \vs{.1cm}
There is a substantial volume of research on both theoretical and computational aspects of parametric SDE models and its statistical inference,  a very limited list of references for which is \cite{Chib01, RoSt01, Kut03, GoWi05, Bis08, GoWi08, ArOp11, CsOp13, BPRF06, FPR08, SuGa16, WGBS17}).   Specifically, for these models the driving functions of the SDE are assumed to be known barring a finite-dimensional parameter $\theta$, which then needs to be estimated from the available data. In reality, for a large class of physical systems functional or parametric forms of the underlying SDEs are not precisely known. However to get a workable mathematical model a heavy set of assumptions is usually imposed on the system which in many cases is not practical --- the resulting model might be too simplistic and might only work in certain ideal situations. For example, in biochemical systems, under a set of assumptions including spatial homogeneity, the intensity function of each reaction driving the stochastic dynamics is assumed to be of the form of  a known polynomial function  multiplied by the corresponding reaction rate constant (unknown parameter). Model calibration then requires estimation of these  reaction rates from the given data \cite{GoWi06, BoWiKi08, GoWi11,KZG12}.  However, most cellular reactions do not occur in spatially homogeneous environments. Moreover there are often many unknown factors (e.g, undiscovered reactions or species) affecting the reaction rates -- assuming that they are constants despite these can lead to simplistic models which might not be able to explain observed behavior of these systems satisfactorily. This highlights the importance of developing truly data-driven models, where  some of the key driving functions (like $b$, $\s$) for a complex  dynamical system of the form \eqref{eq:sde0} are learnt entirely from the given data.
 
  \vs{.1cm}	
  This  however is a hard infinite-dimensional learning problem! Compared to the parametric case, very little is available in the literature for these nonparametric stochastic models. Most of the research in the area of machine learning and `traditional' nonparametric statistics focus on regression or classification analysis involving i.i.d data points, which are comparatively much easier to work with. For stochastic dynamical systems that we are interested in,  there exist some histogram based approaches using bins of size $\ep$ around each location $x$ and computing appropriate local means in those bins \cite{FPST11}. Further refinements include replacing the bins with means of $k$-nearest neighbor \cite{HeSt09} and use of traditional Nadaraya-Watson type estimates \cite{LaKl09}. These methods unfortunately only work for a limited number of toy systems and require high number of data-points around each $x$. Some approaches involving Gaussian Process \cite{RBO13, Yil18} have also been used, but they often rely on adhoc approximation including linearization which might not be desirable.   
  
   \vs{.1cm}	
  The present paper along with related future projects aims to develop a systematic Bayesian framework for addressing these types of complex problems. The data for these problems  can come in a wide array of formats --- ranging from a single path observed at high frequency to noisy partial observations observed at sparse times. This article is the first in the series of ongoing and planned papers \cite{GaMiZh22-a, GaMiZh22-b} that aims to develop learning schemes for these  different data settings. This article specifically  focuses on learning of the drift function of SDEs in the case of high frequency data by which we mean that it is of the form of a single discrete path $\{X(t_i):i=1,2\hdots,m\}$ where the gap $t_i-t_{i-1}$ between two successive observation times $t_i$ and $t_{i-1}$ is very small. Our first step toward estimating the driving functions of the SDE is to consider the problem of minimization of the negative log-likelihood subject to a penalty function over an appropriate function space. Reproducing kernel Hilbert spaces (RKHS) are most suitable function spaces for these kinds of infinite-dimensional optimization problems because of the well-known representer theorem which often converts a class of such problems into finite-dimensional ones. However the limitation of the representer theorem is that it requires the loss functional, $L(h)$, to depend on the input function $h$ only through its values, $h(x_i)$, at a  finite number of data points $\{x_i\}$ which makes it or its known variants inapplicable in many important cases. 
  
   \vs{.1cm}	
  This issue is addressed in the first part of the paper (Section \ref{sec:optH}), which studies infinite-dimensional optimization problems in a broader framework and proves certain general results (see Theorem \ref{th:optH} and its corollaries), special cases of which give the representer theorem on RKHS.  Results of Section \ref{sec:optH} should be of independent interest and are expected to find wider applications. The full generality of Theorem \ref{th:optH} is crucial in our upcoming papers involving more general stochastic models; in the current paper, only a slightly generalized version of the representer theorem is needed and it gives a representation of the minimizer of the penalized negative log-likelihood in an RKHS as a finite-sum with respect to the basis-functions, $\knl(\cdot, X(t_i))$, where $\knl$ is the associated kernel of the RKHS. We next develop a Bayesian hierarchical framework for estimating the coefficients of this finite-sum representation by putting appropriate prior distributions on them. The primary advantage of the Bayesian approach over point-optimization methods (like gradient descent) is the proper quantification of uncertainty through the posterior distributions of the estimators. Now the number of terms in this finite-sum expansion increases proportionately with the number of data points. It is therefore imperative that sparse learning is incorporated to reduce the complexity of the estimators. In our Bayesian paradigm, this is induced through proper shrinkage priors, and in this paper we employ a multivariate $t$-prior and an extension of Horseshoe like priors for this purpose. The  interplay of shrinkage priors and the SDE dynamics is interesting to note. Shrinkage priors are effective in case of positive recurrence which forces the SDE to revisit the relevant parts of the state space numerous times over a finite time horizon. This implies that not all of the basis functions $\knl(\cdot, X(t_i))$ are needed in the finite-sum expansion of the estimator of the drift function;  only a limited selection is enough for accuracy, and  proper shrinkage priors help to identify this selection.  The use of shrinkage priors in the context of SDEs is novel and to the best of our knowledge has not been studied before.
  
  \vs{.1cm}
  The layout of the article is as follows. Section \ref{sec:optH} studies optimization problem in the setting of a general Hilbert space. Section \ref{sec:SDE-model} introduces the SDE model and formulates the Bayesian framework with shrinkage priors for learning the drift function. The learning algorithms are also presented. Numerical examples are discussed in Section \ref{sec:sim}. Finally, some concluding remarks can be found in Section \ref{sec:dis}.

\vs{.2cm}
\np
{\em Notation:} $\R^{m\times n}$ denotes the space of $m\times n$ real matrices. $\ve_{m\times n}: \R^{m\times n} \rt \R^{mn}$ will denote the vectorization function for $m\times n$ matrices. For two Hilbert (or Banach) spaces $\Hsp_1$ and $\Hsp_2$, $L(\Hsp_1,\Hsp_2)$ denotes the space of  linear bounded operators from $\Hsp_1$ to $\Hsp_2$.  $\Hsp_1\oplus_e\Hsp_2$ will denote the external direct sum of $\Hsp_1$ and $\Hsp_2$. $\No_d(\mu, \Sigma)$ will refer to the $d$-dimensional Normal distribution with mean $\mu$ and covariance matrix $\Sigma$, and for notational convenience $\No_d(\cdot| \mu, \Sigma)$ will denote the corresponding density function. Similar convention will be followed for other named distributions:
\begin{itemize}
\item $\tdst_d(\nu, \mu, V)$:\  $d$-dimensional $\tdst$-distribution with degrees of freedom $\nu$, mean $\mu$ and scale matrix $V$; $\tdst_d(\cdot|\nu, \mu, V)$:\ corresponding density function (c.f \eqref{eq:den-t}).
\item $\SC{G}(a,b),\ \SC{IG}(a,b)$: Gamma and Inverse Gamma  distributions with parameters $a$ and $b$;\\ $\SC{G}(\cdot|a,b),\ \SC{IG}(\cdot|a,b)$: corresponding density functions.
\item $\SC{W}_d(\nu,V),\  \SC{IW}_d(\nu,V)$:\  $d$-dimensional Wishart and Inverse-Wishart distributions with degrees of freedom $\nu >d-1$ and scale matrix $V$;\\
$\SC{W}_d(\cdot|\nu,V)$ and $\SC{IW}_d(\cdot|\nu,V)$:\  the corresponding density functions. 
\item $\Fdst(\nu_1,\nu_2,c)$:\ $\Fdst$-distribution  with degrees of freedom $\nu_1, \nu_2$ and scaling parameter $c$;\\ $\Fdst(\cdot|\nu_1,\nu_2,c)$:\ the corresponding density function (c.f \eqref{eq:den-F})
\end{itemize}

\setcounter{equation}{0}
\section{Optimization in Hilbert space} \label{sec:optH}
Let $F :\SC{H} \times [0,\infty) \rt \R$. We are interested in the minimization problem
\begin{align}\label{prob:min}
	\min_{h \in \SC{H}} F\lf(h, \<Qh, h\>^{1/2}\ri).
\end{align}
where $Q \in L(\Hsp,\Hsp)$ is a self-adjoint, positive semidefinite (p.s.d) continuous linear operator.  Notice that this class of minimization problems is equal to the class of problems of the type $	\min_{h \in \SC{H}} F\lf(h, \|Rh\|\ri),$ where $R \in L(\Hsp,\Hsp)$.

Recall that $Q  \in L(\SC{H}, \SC{H}) $ is  {\em positive} or {\em positive semi-definite} (p.s.d)  if $ \<Qh,h\> \geq 0$ for any $h \neq 0$,  {\em positive definite} (p.d)  if the previous inequality is strict  for all $h\neq 0$, and {\em uniformly positive definite} (uniformly p.d.) if there exists a $\l>0$ such that $\<Qh,h\>  \geq \l \|h\|^2$ for all $h \in \SC{H}$.  If $H_0$ is a subset of $\Hsp$, then the restriction of $Q$ to $H_0$, $Q\big|_{H_0}$, is p.s.d (p.d) if $ \<Qh,h\>  \geq 0\ (>0)$ for any $0\neq h \in H_0$, and {\em uniformly p.d} if for some $\l>0$,  $\<Qh,h\>  \geq \l \|h\|^2$ for all $h \in H_0$.

For $Q \in  L(\Hsp,\Hsp)$, define
\begin{align}\label{eq:nullQ}
\SC{N}_Q = \lf\{h \in \Hsp: \<Qh,h\>=0\ri\}
\end{align}	
Clearly, if $Q$ is self-adjoint and p.s.d, $\SC{N}_Q$ is a closed subspace of $\Hsp$, and a p.s.d operator $Q$ is p.d if and only if $\SC{N}_Q =\{0\}.$ Further note that if $\SC{M}$ is a subspace of $\Hsp$, then $Q\big|_\SC{M}$ is p.d if $\SC{N}_Q \cap \SC{M} =\{0\}.$ When $Q$ is p.s.d, $h \rt \<h,Qh\>^{1/2}$ defines a seminorm; it is a proper norm when $Q$ is p.d, in which case we write $\|h\|_Q \equiv \<h,Qh\>^{1/2}$.  $\|\cdot\|_Q$ is equivalent to the original $\|\cdot\|$ norm if and only if $Q$ is uniformly p.d.

By a solution to the problem \eqref{prob:min} we will mean a {\em (global) minimizer} $h^* \in \SC{H}$ such that 
$$F\lf(h^*, \<Qh^*, h^*\>^{1/2}\ri) = \inf_{h \in \SC{H}} F\lf(h, \<Qh, h\>^{1/2}\ri) \stackrel{def}= F^*.$$
In contrast, an element $h_0 \in \SC{H}$ is a {\em local minimizer} of the problem  \eqref{prob:min}  if there exists a $r>0$, such that $F\lf(h_0, \<Qh_0, h_0\>^{1/2}\ri) = \inf_{h \in B(h_0, r)} F\lf(h, \<Qh, h\>^{1/2}\ri).$
Here $B(h_0,r)$ is the open ball in $\SC{H}$ with center at $h_0$ and radius $r$.

 The following useful  result which characterizes uniformly p.d operators is standard and also easy to show. A proof is given in the Appendix for completeness.
 \begin{lemma} \label{lem:pd-op}
 Let $Q \in L(\Hsp,\Hsp) $ be a self-adjoint, p.d. operator. Then the following are equivalent:
 \begin{center}
 (i)\ $Q$ is uniformly p.d. \quad (ii)\  $\ran(Q)$ is  closed. \quad (iii)\  $Q$ is surjective.
 \end{center}
 
 \end{lemma} 

It is useful to note here that by the Hellinger-Toeplitz theorem (or simply by the closed graph theorem) if $Q:\Hsp \rt \Hsp$ is a self-adjoint  linear operator  with Dom$(Q) = \H$, then $Q$ has to be continuous, that is, $Q \in L(\Hsp,\Hsp)$.  

Lower semicontinuity (l.s.c) plays an important role in the solution of a minimization problem. Since there are different notions of l.s.c in a Hilbert space, we first recall their definitions.

\begin{definition}
\begin{enumerate}[(i)]
\item A function $G: \SC{H} \rt [-\infty, \infty]$ is said to be strongly lower-semicontinuous (l.s.c) or l.s.c in the norm topology if $\liminf_{n\rt \infty} G(h_n) \geq G(h)$, whenever $h_n \rt h$ (in $\SC{H}$-norm);  or equivalently, the sublevel sets $\{h: G(h) \leq a\}$ are closed in the norm topology of $\SC{H}$.
\item A function $G: \SC{H} \rt [-\infty, \infty]$ is said to be weakly sequentially l.s.c  if $\liminf_{n\rt \infty} G(h_n) \geq G(h)$, whenever $h_n \stackrel{w}\rt h$, or equivalently, the sublevel sets $\{h: G(h) \leq a\}$ are weakly sequentially closed.
\item A function $G: \SC{H} \rt [-\infty, \infty]$ is said to be weakly l.s.c  if  the sublevel sets $\{h: G(h) \leq a\}$ are closed in the weak topology on $\SC{H}$.

\item If $H_0 \subset \SC{H}$, the restriction of $G$ to $H_0$, $G\big|_{H_0}$, is weakly sequentially l.s.c if $\liminf_{n\rt \infty} G(h_n) \geq G(h)$ whenever $\{h, h_n, n\geq 1\} \subset H_0$ and $h_n \stackrel{w}\rt h$ in $H_0$ in the sense for any $g \in H_0$, $\<h_n,g\> \rt \<h,g\>$. Strong  l.s.c of $G\big|_{H_0}$ is defined similarly.
\end{enumerate}
\end{definition}	 
All notions of l.s.c are equivalent when $\SC{H}$ is finite-dimensional, but that is obviously not the case when $\SC{H}$ is infinite-dimensional. For infinite-dimensional Hilbert spaces, it should be noted that the notion of weakly sequentially l.s.c is not equivalent to that of weakly l.s.c (since the weak topology on $\SC{H}$ is not metrizable). In fact, we have the following hierarchy:
\begin{center} $G$ is weakly l.s.c \ $\RT$\ $G$ is weakly sequentially l.s.c \ $\RT$\ $G$ is strongly l.s.c. 
\end{center}	
This is immediate because  a subset $C \subset \SC{H}$ is weakly closed \ $\RT$\ $C$ is weakly sequentially closed \ $\RT$\ $C$ is closed in the norm topology. Thus the assumption of strong l.s.c on a function $G$ is a weaker assumption than that of weak l.s.c of $G$. However, under the additional assumption of quasiconvexity, all notions of l.s.c are equivalent (see Remark \ref{rem:quasi}-(iii) below). 

\begin{definition} \label{def:quasi}
A function $G : \SC{H} \rt [-\infty, \infty]$ is {\em quasiconvex} if for any $\delta \in [0,1]$ and $h,h' \in \Hsp$,
\begin{align}\label{eq:quasi-conv}
G(\delta h+(1-\delta) h') \leq \max\{G(h), G(h')\},
\end{align}
or equivalently, the sublevel sets $\{h: G(h) \leq a\}$ are convex. It will be called {\em almost quasiconvex}, if \eqref{eq:quasi-conv} holds for $0<\delta<1$ when $G(h) \neq G(h')$. 

$G$ is {\em strictly quasiconvex} if the inequality in \eqref{eq:quasi-conv} is strict for $0<\delta<1$ and $h \neq h'$. It will be called  {\em almost strictly quasiconvex} if  \eqref{eq:quasi-conv} holds with strict inequality for $G(h) \neq G(h')$ and $0<\delta<1$.

\end{definition}	

Note that for almost quasiconvex or almost strictly quasiconvex functions no stipulations are made if $G(h) = G(h')$.

\begin{remark} \label{rem:quasi} {\rm
\text{}
\begin{enumerate}[(i)]	
\item 
 The definition of strict quasiconvexity is not uniform in the literature. Slight variants of the definition given above have been used in the literature. In particular, \cite{GrPi71} used strict quasiconvexity for functions which we call here almost strictly quasiconvex.
 
 \item A strictly quasiconvex function is of course quasiconvex, and an almost strictly quasiconvex function is almost quasiconvex. But an almost strictly quasiconvex function need not be quasiconvex. The standard example given in \cite{GrPi71} is $G:\R \rt \R$ defined by $G(x) = 1_{\{0\}}(x)$. It's clear $G$ is  almost strictly quasiconvex, but the sublevel set $\{x: G(x) \leq 0\} = \R-\{0\}$, which is not convex; hence $G$ is not quasi-convex.

 \item A convex function is of course both quasiconvex and almost strictly  quasiconvex, and a strictly convex function is strictly quasiconvex. The equivalence of strong and weak l.s.c of a function $G:\Hsp \rt [-\infty, \infty]$ under the assumption of quasiconvexity is simply a consequence of Mazur's lemma which, in particular, states that a convex subset $C \subset \SC{H}$ is closed in the norm topology iff it is closed in the weak topology.
 
 \end{enumerate}
}
 \end{remark}

 \begin{lemma}\label{lem-opt}
Let $G:\Hsp \rt [-\infty, \infty]$ be weakly sequentially l.s.c, and  $\limsup_{\|h\| \rt \infty}G(h) = \infty.$ Then there exists a global minimizer $h^* \in \Hsp$ such that $G(h^*) = \min_{h \in \Hsp} G(h) = \inf_{h \in \Hsp} G(h).$ 
\end{lemma} 

\begin{proof}
 Define $G^* \stackrel{def} =\inf_{h \in \SC{H}} G(h)$
 and  observe that if $G \equiv \infty$, the assertion is trivially true as then $G^* = \infty$,
and any $h \in \SC{H}$ solves the minimization problem. So we assume that $G(h)< \infty$ for some $h \in \SC{H}$. Then $G^* < \infty$ ($G^*$ still could be $-\infty$), and there exists a sequence $\{h_n\}$ such that 
$G(h_n) \rt G^*$, as $ n \rt \infty.$
 Notice that this implies the sequence $\{\|h_n\|\}$ is bounded. Indeed, if this is not true then $\limsup_{n \rt \infty} \|h_n\| = \infty$. But the hypothesis on $G$ then implies that $G^* = \limsup_{n\rt \infty} G(h_n) = \infty$, which contradicts the fact that $G^*<\infty$. Consequently, by Banach-Alaoglu (and Eberlein-Smulian theorem) there exists an $h^* \in \SC{H}$ and a subsequence $\{n_k\}$ such that $h_{n_k} \stackrel{w} \rt h^*$. By the weak sequential l.s.c of $G$ we conclude
 \begin{align*}
	G^* = \lim_{k\rt \infty} G(h_{n_k}) \geq G(h^*) \geq \inf_{h \in \SC{H}} G(h) = G^*.
\end{align*}
This proves that the infimum of $G$ is attained at $h^*$.
\end{proof}

\begin{theorem}\label{th:optH}
	Let $\Hsp$  be a Hilbert space, and $F:\SC{H} \times [0,\infty)  \rt [-\infty, \infty]$. Consider the 
	minimization problem \eqref{prob:min}
	where $Q \in L(\Hsp, \Hsp)$ is  self-adjoint and p.s.d. Let  $\SC{M}$ be a closed subspace of $\Hsp$, and the following conditions hold:  (a) $F(h,u) \geq F(\SC{P}_{\SC{M}}h, u)$, $h  \in \SC{H}, u \in [0,\infty),$ where $\SC{P}_{\SC{M}} : \Hsp \rt \SC{M}$ is the (orthogonal) projection operator onto the subspace $\SC{M}$, (b) for each fixed $h \in \Hsp$, the mapping $u \in \R \Rt F(h, u)$ is non-decreasing, and  (c) $Q\SC{M} \subset \SC{M}$.
	
	\begin{enumerate}[(i)]
		
		\item
		Then $\inf_{h \in \SC{H}} F\lf(h ,\<Qh ,h\>^{1/2}\ri) = \inf_{h  \in \SC{M}} F\lf(h ,\<Qh,h\>^{1/2}\ri)$. If $h^* \in \Hsp$ is a global minimizer of $F\lf(h ,\<Qh,h\>^{1/2}\ri)$, then so is $\SC{P}_{\SC{M}}h^*$; in other words existence of a minimizer also guarantees existence of a minimizer lying in $\SC{M}$. If in addition for each $h\in \Hsp$, the mapping $u \rt F(h, u)$ is strictly increasing and $\SC{N}_Q \subset \SC{M}$ (or equivalently, $\SC{N}_Q\cap \SC{M}^\perp = \{0\}$), then any (global) minimizer $h^*$ of the minimization problem (when it exists) lies in $\SC{M}$. 
		
		\item If for each fixed $u \in [0,\infty)$, the mapping $h \in \Hsp \Rt F(h,u)$ is almost quasiconvex, and for each $h\in \Hsp$, the mapping $u \rt F(h, u)$ is strictly increasing and $\SC{N}_Q \subset \SC{M}$  , then any local minimizer $h^0$ of \eqref{prob:min} (when it exists) lies in $\SC{M}$.
		
		\item If $F$ is almost strictly quasiconvex (in particular, convex), then any local minimizer $h^0$ is also a global minimizer. If $F$ is strictly quasiconvex, then the global minimizer of $F$, when it exists, is unique.

	\end{enumerate}
	
\end{theorem}

\begin{proof}
	\np	
	(i) Fix $h \in \SC{H}$. Write $h = \SC{P}_{\SC{M}}h+(I-\SC{P}_{\SC{M}})h$. Next notice that since $Q$ is self-adjoint,
	\begin{align*}
		\<Qh, h\> =&\ \<Q\SC{P}_{\SC{M}}h + Q(I-\SC{P}_{\SC{M}})h, \SC{P}_{\SC{M}}h + (I-\SC{P}_{\SC{M}})h \>\\
		=&\ \<Q\SC{P}_{\SC{M}}h,  \SC{P}_{\SC{M}}h \> + 2\<Q(I-\SC{P}_{\SC{M}})h, \SC{P}_{\SC{M}}h\>+ \<Q(I-\SC{P}_{\SC{M}})h, (I-\SC{P}_{\SC{M}})h\>\\
		=&\ \<Q\SC{P}_{\SC{M}}h,  \SC{P}_{\SC{M}}h \> + \<Q(I-\SC{P}_{\SC{M}})h, (I-\SC{P}_{\SC{M}})h\>
	\end{align*}
because $\<Q(I-\SC{P}_{\SC{M}})h, \SC{P}_{\SC{M}}h\> = \<(I-\SC{P}_{\SC{M}})h, Q\SC{P}_{\SC{M}}h\> =0$,
	  as $(I-\SC{P}_{\SC{M}})h \in \SC{M}^\perp$ and $Q\SC{P}_{\SC{M}}h \in \SC{M}$ (because of the hypothesis, $Q\SC{M} \subset \SC{M}$).  
	Since $Q$ is p.s.d, it follows that $\<Qh, h\> \geq \<Q\SC{P}_{\SC{M}}h,  \SC{P}_{\SC{M}}h \> $ with equality only when $(I-\SC{P}_{\SC{M}})h \in \SC{N}_Q$.
	
	Since $F(h, \cdot) \geq  F(\SC{P}_{\SC{M}}h, \cdot)$ and $F(h,\cdot)$ is non-decreasing, we have 
	\begin{align}\label{ineq:F-0}
	F\lf(h ,\<Qh,h\>^{1/2}\ri) \geq F\lf(\SC{P}_{\SC{M}}h,  \<Qh, h\>^{1/2}\ri) \geq F\lf(\SC{P}_{\SC{M}}h,  \<Q\SC{P}_{\SC{M}}h,  \SC{P}_{\SC{M}}h \>^{1/2}\ri).
	\end{align}
   This proves both the first and the second assertions of (i). 
	If $F(h,\cdot)$ is strictly increasing, then the second inequality in \eqref{ineq:F-0} is strict when $(I-\SC{P}_{\SC{M}})h \notin \SC{N}_Q$. Now $(I-\SC{P}_{\SC{M}})h \in \SC{M}^\perp$.  Therefore, if $\SC{N}_Q \subset \SC{M}$, or equivalently, $\SC{N}_Q \cap \SC{M}^\perp = \{0\}$, then the second inequality in \eqref{ineq:F-0} is strict if and only if $h \neq \SC{P}_{\SC{M}}h$.   Consequently, if $h^* \in \Hsp$, is a global minimizer of \eqref{prob:min},  we must have   
	 $h^* = \SC{P}_{\SC{M}}h^*$, or equivalently, $h^* \in \SC{M}$. This proves the last part of (i).\\
	
	\np
	(ii) We prove the statement by contradiction. Let $h_0$ be a local minimizer. Then there exists a $r>0$ such that $F\lf(h_0 ,\<Qh_0,h_0\>^{1/2}\ri) \leq F\lf(h ,\<Qh,h\>^{1/2}\ri) $ for all $h \in B(h_0,r)$. Suppose that $h_0 \notin \SC{M}$. Then $h_0 \neq \SC{P}_{\SC{M}}h_0$. Consequently, by the previous proof 
	$\<Qh_0, h_0\> > \<Q\SC{P}_{\SC{M}}h_0,  \SC{P}_{\SC{M}}h_0 \>$. 
	For $0\leq \delta \leq 1$, define $h_\delta = \delta \SC{P}_{\SC{M}}h_0+(1-\delta)h_0$. Note that by convexity of the mapping $h \rt \<Qh, h\>^{1/2}$, for any $0<\delta<1$,
	\begin{align*}
		\<Qh_\delta,h_\delta\>^{1/2} \leq \delta \<Q\SC{P}_{\SC{M}}h_0, \SC{P}_{\SC{M}}h_0\>^{1/2} +(1-\delta)\<Qh_0,h_0\>^{1/2} < \<Qh_0,h_0\>^{1/2}.
	\end{align*}
	Thus for any $0< \delta < 1$ by almost quasiconvexity of $F(\cdot, u)$ (c.f. Definition \ref{def:quasi}),
	\begin{align}
		\non
		F\lf(h_\delta, \<Qh_\delta, h_\delta\>^{1/2}\ri) <&\ F\lf(h_\delta, \<Qh_0,  h_0 \>^{1/2}\ri) \leq F\lf(h_0, \<Qh_0,  h_0 \>^{1/2}\ri) \vee F\lf(\SC{P}_\SC{M}h_0, \<Qh_0,  h_0 \>^{1/2}\ri)\\
		\label{eq:F-quasi}
		 =&\ \ F\lf(h_0, \<Qh_0,  h_0 \>^{1/2}\ri).
	\end{align}	 
	The last equality is because $F\lf(h_0, \<Qh_0,  h_0 \>^{1/2}\ri) \geq F\lf(\SC{P}_\SC{M}h_0, \<Qh_0,  h_0 \>^{1/2}\ri)$ due to the assumption on $F$. Now notice that $\|h_\delta - h_0\| = \delta \|(I-\SC{P}_{\SC{M}})h_0\| < r$ for sufficiently small $\delta$, and hence $F\lf(h_0, \<Qh_0,  h_0 \>^{1/2}\ri) \leq F\lf(h_\delta, \<Qh_\delta, h_\delta\>^{1/2}\ri)$ for sufficiently small $\delta$. But that is a contradiction to \eqref{eq:F-quasi}.\\
	
%
	\np
	(iii) is essentially a standard result in convex optimization.
%
%
\end{proof}

\begin{remark} {\rm
If $Q$ is self-adjoint and $\SC{M}$ is a closed subspace then
 $Q\SC{M} \subset \SC{M}$ (see condition (c) in Theorem \ref{th:optH}) is equivalent to $Q\SC{M}^\perp \subset \SC{M}^\perp$ which in turn is equivalent to commutativity of $Q$ and $\SC{P}_{\SC{M}}$. The first equivalence is easy to see. It is also immediate that if $Q$ and $\SC{P}_{\SC{M}}$ commute, then  $Q\SC{M} \subset \SC{M}$. To see the other direction of the second equivalence, we have for any $h \in \Hsp$
$$Q\SC{P}_{\SC{M}}h+Q(I-\SC{P}_{\SC{M}})h=Qh = \SC{P}_{\SC{M}}Qh+(I-\SC{P}_{\SC{M}})Qh.$$
Now  $Q\SC{M} \subset \SC{M}$ and $Q\SC{M}^\perp \subset \SC{M}^\perp$ imply that $Q\SC{P}_{\SC{M}}h \in \SC{M}$ and $Q(I-\SC{P}_{\SC{M}})h \in \SC{M}^\perp$, and, of course, by the definition of $\SC{P}_{\SC{M}}$, $\SC{P}_{\SC{M}}Qh   \in \SC{M}$ and $(I-\SC{P}_{\SC{M}})Qh \in \SC{M}^\perp$. Since $\Hsp = \SC{M}\oplus\SC{M}^\perp$, we must have $\SC{P}_{\SC{M}}Qh = Q\SC{P}_{\SC{M}}h$.
}
\end{remark}	

 In many applications $F$ is of the form $F(h,u) = F_0(h) + J(u)$, where $F_0$ can be viewed as a loss function and an associated penalty function on the size of $h$ is defined through $J$. A typical choice of $J$ and the operator $Q$ are $J(u)=u^2$, $Q =I$, which defines the popular square-norm penalty function, $\|h\|^2$.  The following corollary is essentially a restatement of Theorem \ref{th:optH} in this case. Importantly, Theorem \ref{th:optH} or Corollary \ref{cor:optH} below  allows use of seminorms $\<h,Qh\>^{1/2}$ which are different from  the original $\Hsp$-norm inside the penalty function $J$. Since $Q$ does not need to be uniformly p.d or even p.d, they are not necessarily equivalent to the $\Hsp$-norm.

\begin{corollary}\label{cor:optH}
Suppose  $F$ is of the form  $F(h,u) = F_0(h) + J(u)$, where $J: [0,\infty) \rt [0,\infty)$ is  strictly increasing. Consider the minimization problem \eqref{prob:min}, and assume the setup of Theorem \ref{th:optH}. In other words, assume that the linear operator $Q: \Hsp \rt \Hsp$ of \eqref{prob:min} is self-adjoint and p.s.d,  $\SC{N}_Q \cup Q\SC{M} \subset \SC{M}$, where $\SC{M}$ is a closed subspace of $\Hsp$ and  $F_0(h) \geq F_0(\SC{P}_{\SC{M}}h)$ for all $h  \in \SC{H}$. 
Then the set of global minimizers,
\begin{align}\label{eq:set-min}
	M_0 \stackrel{def}= \lf\{h^* \in \SC{H}: F\lf(h^*, \<Qh^*, h^*\>^{1/2}\ri) = F^*  =\inf_{h \in \SC{H}} F\lf(h, \<Qh, h\>^{1/2}\ri)\ri\} \subset \SC{M}.
\end{align} 

Suppose in addition $F_0\big|_{\SC{M}}: \SC{M} \rt \Hsp$ is weakly l.s.c, $J$ is l.s.c  and either (a) $F_0\big|_{\SC{M}}$ is bounded below, $J$ is coercive (that is, $\limsup_{u\rt \infty} J(u) = \infty$), and $Q\big|_{\SC{M}}$ is uniformly p.d (in particular, $Q$ is p.d because of the eariler assumption $\SC{N}_Q \subset \SC{M}$) , or (b) $ \limsup_{h\in \SC{M},\ \|h\| \rt \infty} F_0(h) = \infty$. Then $M_0 \neq \emptyset$.

\end{corollary}	


\begin{proof}
\eqref{eq:set-min} follows from Theorem \ref{th:optH}-(i). The fact that $M_0$ is nonempty (existence of minimizer) is a direct consequence of Lemma \ref{lem-opt} applied in the setting of Hilbert subspace $\SC{M}$ (recall that $\SC{M}$ is closed). To see this we start by noting that the mapping $h \in \SC{M} \rt \<Qh,h\>^{1/2}$ is weakly l.s.c. This is because the sublevel sets $\lf\{h \in \SC{M}: \<Qh,h\>^{1/2} \leq a\ri\} = \lf\{h \in \SC{M}: \<Qh,h\> \leq a^2\ri\}$ are weakly closed since they are strongly closed (as the mapping $h \in \SC{M} \rt \<Qh,h\>$ is strongly continuous) and convex (due to convexity of $h \in \SC{M} \rt \<Qh,h\>$). Since $J: [0,\infty) \rt [0,\infty)$ is l.s.c and increasing, the (composition) mapping $h \in \SC{M} \rt J\lf(\<Qh,h\>^{1/2}\ri)$ is also weakly l.s.c. Hence, because of the hypothesis that $F_0$ is weakly sequentially l.s.c,  the mapping $h\in \SC{M} \rt F\lf(h, \<Qh,h\>^{1/2}\ri)$ is weakly sequentially l.s.c. 

Now clearly (b) implies that $\limsup_{h\in \SC{M},\ \|h\| \rt \infty} F\lf(h,\<Qh,h\>^{1/2}\ri) = \infty$.  If (a) holds instead of (b), then we just need to observe that $\limsup_{h\in \SC{M},\ \|h\| \rt \infty} J(\<Qh,h\>^{1/2}) = \infty$. This follows as for some constant $\l>0$, $\<Qh,h\> \geq \l \|h\|^2$ for all $h \in \SC{M}$ (as $Q\big|_{\SC{M}}$ is uniformly p.d.)  and $\limsup_{u\rt \infty} J(u) = \infty$. Since $F_0$ is bounded below, 
 $\limsup_{h\in \SC{M},\ \|h\| \rt \infty} F(h, \<Qh,h\>^{1/2}) = \infty.$ In either case, the assertion   follows from Lemma \ref{lem-opt}.
\end{proof}

In many applications it is desirable to consider minimization problems where penalty is imposed on the size of only a part of the function $h$. Below we demonstrate that Corollary \ref{cor:optH} covers such cases. In machine-learning, such minimization problems arise when partial structure of the unknown function $h$ to be learned is known, and the so-called semiparametric representer theorem (which is a special case of Corollary \ref{cor:optH} or Corollary \ref{cor:optH-semi} below) is a useful result covering a subset of such instances.\\

For two Hilbert spaces $(\Hsp_1, \<\cdot,\cdot\>_1)$ and $(\Hsp_2, \<\cdot,\cdot\>_2)$, recall that the external direct sum  $\Hsp_1 \oplus_e \Hsp_2$ is the space $\Hsp_1\times\Hsp_2$ equipped with the inner product
$$\<(h_1,h_2), (h_1',h_2')\>_e = \<h_1,h_1'\>_1+\<h_2,h_2'\>_2.$$

\begin{corollary}\label{cor:optH-semi}
Let  $(\Hsp_1, \<\cdot,\cdot\>_1)$ and $(\Hsp_2, \<\cdot,\cdot\>_2)$ be two Hilbert spaces and $\Hsp = \Hsp_1 \oplus_e \Hsp_2 $. Suppose  $F$ is of the form  $F(h,u) = F_0(h) + J(u)$, where $J: [0,\infty) \rt [0,\infty)$ is  strictly increasing. Consider the minimization problem 
\begin{align} \label{prob:min-semi}
\min_{h = (h_1,h_2) \in \SC{H}} F_0(h)+J\lf(h_1, \<Q_1h_1, h_1\>^{1/2}\ri)
\end{align}
 where $Q_1 \in L(\Hsp_1,\Hsp_2)$ is self-adjoint and p.s.d. Let $\SC{M}_1$ be a closed subspace of $\Hsp_1$, and assume that   $\SC{N}_{Q_1} \cup Q_1\SC{M}_1 \subset \SC{M}_1$, $F_0(h)= F_0(h_1,h_2) \geq F_0(P_{\SC{M}_1}h_1, h_2)$ for all $h =(h_1,h_2) \in \SC{H}$.
Then the set of global minimizers, $M_0 \subset \SC{M}_1\oplus_e\Hsp_2$

Suppose in addition $F_0\big|_{\SC{M}_1\oplus_e \Hsp_2}: \SC{M}_1\oplus_e \Hsp_2 \rt \Hsp$ is weakly sequentially l.s.c, $\limsup \limits_{\substack{ \|h\| \rt \infty\\ h\in \SC{M}_1\oplus \Hsp_2}} F_0(h) = \infty$, and $J$ is l.s.c. Then $M_0 \neq \emptyset$.
\end{corollary}

\begin{proof}
Define $\SC{M} = \SC{M}_1 \oplus_e \SC{H}_2$, $Q :\Hsp \rt \Hsp$ by $Qh= Q(h_1,h_2) = (Q_1h_1, 0)$ and notice that $\SC{N}_Q = \SC{N}_{Q_1}\oplus \Hsp_2$. The assertion now follows from Corollary \ref{cor:optH}.
\end{proof}

\begin{remark} \label{rem:optH-fin}
{\rm If $\SC{M}$ is a finite-dimensional subspace of $\SC{H}$, which is an important case in practice, and $F(h,u) = F_0(h)+J(u)$, then strong l.s.c of $F_0\big|_{\SC{M}}$, which is easier to check, is equivalent to weak l.s.c (and hence weak sequential l.s.c) of $F_0\big|_{\SC{M}}$. No additional assumption of quasiconvexity of $F_0\big|_{\SC{M}}$ is needed. Furthermore, in this case $Q\big|_{\SC{M}}$ is p.d iff it is uniformly p.d. Thus the conditions of Corollary \ref{cor:optH} are easier to check.

}
\end{remark}

\subsection*{Classical representer theorem} 
The  representer theorem is a seminal result in learning theory which converts a class of infinite-dimensional optimization problems on an RKHS to a tractable finite-dimensional one. It was first derived by Kimeldorf and Wahba in \cite{KiWa71} for quadratic loss and penalty functions in the setting of Chebyshev splines and was later extended to more general RKHS framework in \cite{Wah90}. Extensions to more general loss and penalty functions have been done in \cite{CoSu90, ScHeSm01} (also see \cite{ScSm01}). Representer theorem for vector-valued functions has been proved in \cite{MiPo05} (also see \cite{AlRoLa12} for a review of results on learning vector-valued functions)

The representer theorem along with most of its extensions is a special case of Theorem \ref{th:optH}. Below we present the generalized semiparametric version of it for vector-valued functions and include conditions for existence. We chose the range of the functions to be finite-dimensional vector space only for ease of presentation, but the same proof (with the appropriate changes) holds if the range of the functions is infinite-dimensional.  

The definition of RKHS of vector-valued functions is very similar to that of the scalar-valued functions with the primary difference being that the associated kernel $\knl$ is now  matrix-valued.

\begin{definition} \label{def:matRK}
Let $\meU$ be an arbitrary space.
A symmetric function $\knl: \meU\times \meU \rt \R^{n\times n}$ is a {\em reproducing kernel} if for any $u, u' \in \meU$, $\knl(u,u')$ is a $n\times n$ p.s.d matrix.

The RKHS associated with a reproducing kernel $\knl$ is a Hilbert space $\Hsp_{\knl}$ of functions $h: \meU \rt \R^n$, such that  for every fixed $u \in \meU$ and a (column) vector $c \in \R^n$, (i) the mapping $u' \rt \knl(u',u)c$ is an element of $\Hsp_{\knl}$, and (ii) $\<h, \knl(\cdot, u)c\> = h(u)^Tc.$

\end{definition}

Property (ii) refers to the reproducing property of the kernel $\knl$ in the vector framework. By an extension of Moore–Aronszajn theorem, given a reproducing matrix-valued kernel $\knl$, a constructive description of the corresponding RKHS $\Hsp_{\knl}$ is given by
$\Hsp_{\knl} = \overline{\text{Span}}\{\knl(\cdot,u): u \in \meU\}.$
Here the overbar denotes closure of a set, and the closure is taken with the norm, $\|\cdot\|_{\knl}$ defined by
$$\|h\|_{\knl} = \sum_{i,j=1}^l c_i^T \knl(u_i,u_j)c_j, \quad h = \sum_{j=1}^l \knl(\cdot, u_j)c_j, \ c_j \in \R^n.$$

\begin{corollary}\label{cor:rep}
Let $L: \R^{nm} \rt [-\infty, \infty]$ be any function, and $J:[0,\infty) \rt [0,\infty)$ nondecreasing, and $\knl: \R^d\times \R^d \rt \R^{n\times n}$ a reproducing kernel. Let $\Hsp_{\knl}$ be the RKHS of functions $h:\R^d \rt \R^n$ corresponding to a symmetric positive definite kernel $\knl$. Let $x_1,x_2,\hdots,x_m \in \R^d$ be fixed. Let $\SC{G} = \text{span}\lf\{\mfk{g}_1,\mfk{g}_2,\hdots,\mfk{g}_r\ri\}$, where $\mfk{g}_1,\mfk{g}_2,\hdots,\mfk{g}_r$ are linearly independent functions mapping $\R^d \rt \R^n$. Consider the objective function
\begin{align*}
 L( h(x_1) + g(x_1),  h(x_2)+g(x_2),\hdots,  h(x_m)+g(x_m))+J(\|h\|), \quad \bar h = (h,g) \in \Hsp_{\knl}\oplus_e \SC{G}
\end{align*}
Then the following hold.
\begin{enumerate}[(a)]
\item If a minimizer to the above objective function exists, then there also exists a minimizer ${\bar h}^*$ of the form
\begin{align}\label{eq:fin}
{\bar h}^*(u) = \lf(\sum_{k=1}^m \knl(x_k,u)c^*_k, \sum_{i=1}^r \mfk{g}_i(u)\alpha^*_i\ri)
\end{align}
for some constants  $c^*_i \in \R^{n}$ and $\alpha_k^* \in \R$. If $J$ is also strictly increasing then any minimizer (when it exists) is of the form \eqref{eq:fin}.

\item If $L$ and $J$ are l.s.c and $L$ is coercive (that is, $\limsup_{\|u\| \rt \infty} L(u) = \infty$), then there exists a minimizer $h^*$  of the form \eqref{eq:fin}.

\end{enumerate}
\end{corollary}

\np
Notice that a Hilbertian structure can be put on $\SC{G}$ with the inner product
\begin{align*}
\<g,g'\>_{\SC{G}}  \dfeq \sum_{i,j=1}^r \alpha_i\alpha_j', \quad g = \sum_{i=1}^r \alpha_i \mfk{g}_i, \quad g' = \sum_{i=1}^r \alpha'_i \mfk{g}_i,
\end{align*}
Define the finite-dimensional subspace
 $$\SC{M} =\lf\{\sum_{i=1}^m\knl(x_i,\cdot)c_i: c_i \in \R^n, i= 1,2,\hdots,m\ri\}$$ 
 and observe  that by the reproducing property for any $h \in \SC{H}_\kappa$ and $v \in \R^n$,
\begin{align*}
h(x_i)^Tv = \<h,\knl(x_i,\cdot)v\>  = \<\SC{P}_{\SC{M}}h,\knl(x_i,\cdot)v\>+ \<(I-\SC{P}_{\SC{M}})h, \knl(x_i,\cdot)v \> = \lf((\SC{P}_{\SC{M}}h)(x_i)\ri)^Tv.
\end{align*}
The second term after the second equality is $0$ because $\knl(x_i,\cdot)v \in \SC{M}$ and $(I-\SC{P}_{\SC{M}})h \in \SC{M}^\perp$. Since the above equality is true for any $v \in \R^n$, it follows that $h(x_i) =\SC{P}_{\SC{M}}h(x_i)$.   Consequently, $F_0(h,g) \stackrel{def}= L(h(x_1)+g(x_1), h(x_2)+g(x_2),\hdots, h(x_m)+g(x_m)) = F_0(\SC{P}_{\SC{M}}h, g)$. Moreover, it is easy to see that for each $x_i$,  $\limsup \limits_{ \|h\| \rt \infty,\ h\in \SC{M}_1} h(x_i) = \infty$ and  $\limsup \limits_{\|g\|_{\SC{G}} \rt \infty} g(x_i) = \infty$, which, because of the hypothesis on $L$, in turn implies that  $\limsup \limits_{\substack{ \|(h,g)\| \rt \infty\\ (h,g)\in \SC{M}\oplus_e \SC{G}}} F_0(h,g) = \infty$. It follows that 
Corollary \ref{cor:rep} is a restatement of Corollary \ref{cor:optH-semi} in this particular case.

\begin{remark}{\rm
\text{}
\begin{enumerate}[(i)]
	\item		
It is obvious that Corollary \ref{cor:rep} covers minimization  the objective function of the form
\begin{align*}
	\tilde L\lf((x_1,y_1,(h+g)(x_1)), (x_2,y_2,(h+g)(x_2)),\hdots, (x_m, y_m, (h+g)(x_m))\ri)+J(\|h\|)
\end{align*}
where the points $(x_i,y_i) \in \R^{d}\times\R^{d'}, i=1,2,\hdots,m$ are fixed. Standard examples include data points $\{(x_i,y_i)\}$ from a regression model, $y=f(x) +\vep$. Indeed, in this case one simply
 defines the function $L:\R^m \rt \R$ in Corollary \ref{cor:rep} as $$L(u_1,u_2,\hdots,u_m) = \tilde L\lf((x_1,y_1,u_1), (x_2,y_2,u_2),\hdots, (x_m, y_m, u_m)\ri).$$

 \item Absence of the semiparametric part as encoded by the space $\SC{G}$ leads to the usual representer theorem. By Corollary \ref{cor:optH} in this case, coercivity of $J$ ($\limsup_{\|u\| \rt \infty} J(u) =\infty$) with lower boundedness of $L$ instead of coercivity of $L$ also guarantees the existence of a minimizer in part (b).  Also as evident from Theorem \ref{th:optH}, seminorms of the form $\<\cdot, Q\cdot\>^{1/2}$, which are different from the RKHS norm, can be used inside $J$.
 
 \item Although informally, one can say that the minimizer in \eqref{eq:fin} is of the form ${\bar h}^*(u) = \sum_{k=1}^m \knl(x_k,u)c^*_k + \sum_{i=1}^r \mfk{g}_i(u)\alpha^*_i$, strictly speaking, such a representation is not correct,  and mathematically it should be represented as a pair as in \eqref{eq:fin}. This is because $\Hsp_{\knl} \cap \SC{G}$ might not be $\{0\}$, in which case the mapping $(h,g) \in \Hsp_{\knl}\oplus_e\SC{G} \rt h+g \in \Hsp_{\knl} + \SC{G} $ is not injective. In other words, the function $f=h+g$ might have  different representations in $\Hsp+\SC{G}$, and consequently the mapping $h+g \rt L( h(x_1) + g(x_1),  h(x_2)+g(x_2),\hdots,  h(x_m)+g(x_m))+J(\|h\|)$ is not a well-defined function! 
 
 
 \end{enumerate}
 
}
\end{remark}	

A common choice of matrix-valued reproducing kernel is the class of separable kernels of the form $(\knl(u,u'))_{i,j} = k(u,u')\rho(i,j)$, where $k$ and $\rho$ are scalar kernels on $\R^d\times\R^d$ and $\{1,2,\hdots,d\}\times \{1,2,\hdots,d\}$, respectively. This is of course same as the class of kernels having the representation $\knl(u,u') = k(u,u')B$ with $B$ being an $n\times n$ p.s.d matrix. Note for most learning problems one can assume without loss of generality that $B=I_n$, as $B$ can be ``absorbed" in the coefficients $c_i$  of the finite expansion of the form \eqref{eq:fin} by redefining $c_i$ as $Bc_i$. More general class of matrix-kernels consists of $\knl$ of the form $\knl(u,u') =\sum_{r=1}^Rk_r(u,u')B_r$. For a given set of data-points $\{x_1,x_2,\hdots, x_m\}$, the associated $nm\times nm$-imensional Gram matrix $\bm{\SC{K}}$, which is important for determination of the coefficients of the finite expansion, is given by $\bm{\SC{K}} =\sum_{r=1}^R\SC{K}_r\ot B_r$. Here $\SC{K}_r = ((k_r(x_i,x_j)))_{m\times m}$ is the usual Gram matrix corresponding to the scalar kernel $k_r$.\\

 \np
 Note that the classical Representer Theorem applies to those optimization problems where the loss function $L$ depends on its argument function $h$ only through $h(x_i), \ i=1,2,\hdots,m$. More convoluted dependence on the function $h$ makes the representer theorem inapplicable. We now mention a few such optimization problems from machine learning which are covered by Corollary \ref{cor:optH} (or more generally Theorem \ref{th:optH}) but where the usual representer theorem cannot be used.
  
\subsection*{Linear functional regression} 
Consider the model
\begin{align}\label{eq:reg-lin}
y = \SC{L}_xh+\vep
\end{align}
 where for each $x$, $\SC{L}_x$ is a linear functional acting on $h$, and $\vep$ captures the noise of the system. Thus here the function $h$ is observed (with errors) through a family of linear functionals. For example, consider the regression model, $Y = h(Z) +\vep$, where $Z$ is not directly observed. Instead for a third random variable $X$, the conditional distribution of $Z|X=x$, $\g(\cdot|x)$, is known (or at least can be well approximated). Integrating the effect of $Z$, the conditional model of $Y$ given $X$ is of the form \eqref{eq:reg-lin}, where for a given $x$,  $\SC{L}_xh = \int h(u) \g(du|x).$

Given data points $\{(x_i, y_i): i=1,2,\hdots,m\}$, the natural approach to learn $h$ is again through the minimization problem of the form
\begin{align}\label{prob:min-lin}
\min_{h \in \SC{H}_{\knl}}	\tilde L\lf((x_1,y_1, \SC{L}_{x_1}h), (x_2,y_2,\SC{L}_{x_2}h),\hdots, (x_m, y_m, \SC{L}_{x_m}h)\ri)+J(\|h\|).
\end{align}
It is clear that the classical representer theorem cannot be applied here directly as the loss function does not depend on $h$ only through the values $h(x_i)$. But Corollary \ref{cor:optH} gives a representation of the minimizer $h^*$. To see this define finite-dimensional vector space 
$$\SC{M} =\text{span}\lf\{f_i: f_i(u) = \SC{L}_{x_i}\knl(u,\cdot), i=1,2,\hdots,m\ri\}.$$
 Here however, we first need to check that $\SC{M}$ is indeed a subspace of $\Hsp_{\knl}$ (as it is not obvious). Nevertheless, it is easy as we first note that by the Riesz representation theorem there is $g_i \in \Hsp_{\knl}$, such that $\SC{L}_{x_i} h = \<h,g_i\>$ for any $h \in \Hsp_{\knl}$. Consequently, 
 $$f_i(u) = \SC{L}_{x_i}\knl(u,\cdot) = \<\knl(u,\cdot),g_i\> = g_i(u)$$
where the last equality is because of the reproducing property. That is $f_i = g_i \in \SC{H}_{\knl}$; hence $\SC{M} \subset \SC{H}_{\knl}$ and $g_i \in \SC{M}$. As before  writing $h \in \SC{H}_{\knl}$ as $h=\SC{P}_{\SC{M}}h+(I-\SC{P}_{\SC{M}})h$, we see that
\begin{align*}
\SC{L}_{x_i}h =&\ \SC{L}_{x_i}\SC{P}_{\SC{M}}h + \SC{L}_{x_i}(I-\SC{P}_{\SC{M}})h = \SC{L}_{x_i}\SC{P}_{\SC{M}}h + \<(I-\SC{P}_{\SC{M}})h, g_i\>   = \SC{L}_{x_i}\SC{P}_{\SC{M}}h.
\end{align*}
The last equality is because $(I-\SC{P}_{\SC{M}})h \in \SC{M}^\perp,$ and we showed that $g_i \in \SC{M}.$ Consequently, $F_0(h) \stackrel{def}= \tilde L\lf((x_1,y_1, \SC{L}_{x_1}h), (x_2,y_2,\SC{L}_{x_2}h),\hdots, (x_m, y_m, \SC{L}_{x_m}h)\ri)= F_0\circ\SC{P}_{\SC{M}}(h)$, and hence by Corollary \ref{cor:optH}  (also see Remark \ref{rem:optH-fin}) a minimizer $h^* \in \SC{M}$; in other words $h^*$ is of the form 
\begin{align*}
h^*(u) = \sum_{i=1}^m \SC{L}_{x_i}\knl(u,\cdot)c_i
\end{align*}

%
%
%

\subsection*{Fredholm integral equation of first kind}
Consider the Fredholm equation of the first kind: $g(x) = \int_{\SC{E}} R(x,u) h(u) du$, where $\SC{E} \subset \R^d$. Here given (possibly noisy) values, $y_i$, of $g$ at finitely many points $x_i$, the goal is to learn the best possible function $h$. The data generating models is thus of the form
$$y_i = \int_{\SC{E}} R(x_i,y) h(y) dy + \ep_i, \quad i=1,2,\hdots,m.$$
where $\ep_i$ captures the noise in the observations. Let $\Hsp_{\knl}$ be the RKHS corresponding to a symmetric, p.d kernel $\knl$, and as before to learn $h$ we consider the minimization problem of the form:
\begin{align}\label{prob:min-fred}
\min_{h \in \SC{H}_{\knl}}	\tilde L\lf((x_1,y_1, Rh(x_1)), (x_2,y_2,Rh(x_2)),\hdots, (x_m, y_m, Rh(x_m))\ri)+J(\|h\|).
\end{align}
where, by a slight abuse of notation, $R$ also denotes the operator /  integral transform corresponding to the kernel $R(\cdot,\cdot)$; that is,
$Rh(x) =  \int_{\SC{E}} R(x,u) h(u) du$.
Note that the classical representer theorem is not applicable as $\tilde L$ depends on $h$, not through values $h(x_i)$ but through the above integrals. But as the following result shows, Corollary \ref{cor:optH} easily gives a representation of the minimizer $h^*$.

\begin{corollary}\label{cor:rep2}
Let $\SC{E} \subset \R^d$ be compact, and let $\krn: \SC{E}\times \SC{E} \rt \R$  be continuous.   Let $\Hsp_{\knl}$ be the RKHS corresponding to a reproducing kernel $\knl$. Assume that $\knl: \SC{E}\times \SC{E} \rt \R$ is continuous.
Let $L: \R^m \rt [-\infty, \infty]$ be any function, and $J:[0,\infty) \rt [0,\infty)$ nondecreasing.  For fixed $x_1,x_2,\hdots,x_m \in \R^d$ consider the objective function
\begin{align*}
 L(\krn h(x_1), \krn h(x_2),\hdots, \krn h(x_m))+J(\|h\|) \stackrel{def}= F(h, \|h\|), \quad h \in \Hsp_{\knl}.
\end{align*}
Then the following hold.
\begin{enumerate}[(a)]
\item If a minimizer to the above objective function exists, then there also exists a minimizer $h^*$ of the form
\begin{align}\label{eq:fin2}
h^*(u) = \sum_{i=1}^m \krn \knl(u, \cdot)(x_i)c_i = \sum_{i=1}^m c_i \int_{\SC{E}} \knl (u,z)\krn(x_i, z)dz 
\end{align}
for some constants $c_i \in \R$. If $J$ is also strictly increasing then any minimizer (when it exists) is of the form \eqref{eq:fin}.

\item Suppose $L$ and $J$ are l.s.c and either (a) $L$ is coercive or (b) $J$ is coercive and $L$ bounded below (for example, non-negative). Then there exists a minimizer $h^*$  of the form \eqref{eq:fin2}.

\end{enumerate}
\end{corollary}

\begin{proof}
It's easy to see that the continuity of the mapping $\knl:\SC{E}\times \SC{E} \rt  \R$  gives continuity of the mapping  $z \in \SC{E} \Rt \knl(z,\cdot) \in \Hsp_{\krn}$. Since $\SC{E}$ is assumed to be compact, the latter mapping is Bochner measurable, and thus so is the mapping $z \in \SC{E} \Rt \knl(z,\cdot) \krn(x_i,z) \in \Hsp_{\krn}$ for each $i=1,2,\hdots,m$ . Moreover, the mapping $z \in \SC{E} \Rt \|\knl(z,\cdot)\krn(x_i,z)\| = \knl(z,z) |\krn(x_i,z)| \in \R$ is obviously integrable  (as it is continuous, and $\SC{E}$ is compact); hence the mapping $z \rt \knl(z,\cdot) \krn(x_i,z)$ is Bochner integrable (e.g. see \cite{Yosi95}). 
Thus the functions $f_i$ defined by the following Bochner integral: 
$$f_i \stackrel{def}= \int_\SC{E} \knl(\cdot,z)\krn(x_i,z)dz$$
are elements of $\Hsp_{\knl}$.
 Since the evaluation functionals are continuous on an RKHS,  obviously,
$f_i(u) =  \int_\SC{E} \knl(u,z)\krn(x_i,z)dz = \krn \knl(u, \cdot)(x_i)$, where the integral in the middle is a regular Riemann integral.

Now define the finite dimensional subspace $\SC{M} \subset \Hsp_{\knl}$ by
$$\SC{M} =\text{span}\lf\{f_i:  i=1,2,\hdots,m\ri\}.$$
Writing $h \in \Hsp_{\knl}$ as $h = \SC{P}_{\SC{M}}h +(I - \SC{P}_{\SC{M}})h$, we see that
\begin{align*}
\krn h(x_i) =&\  (\krn  \SC{P}_{\SC{M}}h)(x_i)+ (\krn(I - \SC{P}_{\SC{M}})h)(x_i)
= (\krn  \SC{P}_{\SC{M}}h)(x_i)+ \int_{\SC{E}}\krn(x_i,z)(I - \SC{P}_{\SC{M}})h(z)\ dz\\
=&\ (\krn  \SC{P}_{\SC{M}}h)(x_i)+ \int_{\SC{E}}\krn(x_i,z)\<(I - \SC{P}_{\SC{M}})h, \knl(z,\cdot)\>\ dz\\
=&\ (\krn  \SC{P}_{\SC{M}}h)(x_i)+ \lf\<(I - \SC{P}_{\SC{M}})h, \int_{\SC{E}}\krn(x_i,z)\knl(z,\cdot)dz\ri\> = (\krn  \SC{P}_{\SC{M}}h)(x_i)+ \lf\<(I - \SC{P}_{\SC{M}})h, f_i\ri\>\\
 =&\  (\krn  \SC{P}_{\SC{M}}h)(x_i),
\end{align*}
where the fourth equality is by the property of Bochner integrals (and the fact that the mapping $g \rt \<(I - \SC{P}_{\SC{M}})h, g\>$ is a continuous linear functional). Consequently, $F_0(h) \stackrel{def}=  L(\krn h(x_1), \krn h(x_2),\hdots, \krn h(x_m))$ $= F_0\circ\SC{P}_{\SC{M}}(h)$, and hence the conclusion of Corollary \ref{cor:rep2}  is just a restatement of Corollary \ref{cor:optH}.


\end{proof}

\np
Relaxations of some of the assumptions including compactness of $\SC{E}$ in Corollary \ref{cor:rep2} are easily possible.

\setcounter{equation}{0}

\section{Framework of stochastic differential equations}\label{sec:SDE-model}
We consider the $d$-dimensional SDE of the form
 \begin{align}\label{eq:sde0}
X(t) = x_0+\int_0^t b(X(s))ds+\int_0^t \s(X(s))dW(s), \quad x_0 \in \R^d,
\end{align}
where $b:\R^d \rt \R^d$ and $\s:\R^d \rt \R^{d\times d}$ and $W$ is a $d$-dimensional Brownian motion. We assume that the functions $b$ and $\s$ are such that the above SDE admits a unique strong solution. This, for example, holds when $b$ and $\s$ are locally Lipschitz and $\s\s^T$ is non singular. The functional forms of $b$ and $\s$ are unknown, and our objective is to learn the SDE, that is, the associated driving functions from high-frequency data $\BX_{t_1:t_m} \dfeq (X(t_1), X(t_2),\hdots, X(t_m))$, where $\Delta=t_i-t_{i-1} \ll 1$. 

Our approach to this problem is to first consider an optimization problem in an appropriate RKHS. 
Assume that for each $t\geq 0$, the distribution of $X(t)$ given $X(0)=x_0$ admits a density $p_t(\cdot|x_0)$ with respect to the Lebesgue measure on $\R^d$. This, for example, exists when for each $x$, $\s\s^T(x)$ is positive definite \cite{RogWil00}. The function $p_t(x|x_0)$ satisfies the Kolmogorov forward PDE (Fokker-Plank equation)
\begin{align*}
\partial_tp_t(x|x_0) = (\SC{L})^*p_t(x|x_0), \quad p_{0}(\cdot|x_0) = \delta_{x_0}
\end{align*}
in weak sense. Here $(\SC{L})^*$ is the adjoint of the generator $\SC{L}$ defined by
$$\SC{L}f(x) =  \sum_{i=1}^d b_i(x) \partial_i f(x)+\f{1}{2}\sum_{1\leq i,j\leq d}(\s\s^T)_{ij}(x)\partial_{ij}f(x), \quad f \in C^2(\R^d, \R).$$
By time-homogeneity, the transition density of $X(t+s)$ given $X(t)=x$ is of course given by $p_s(\cdot| x)$. Therefore the likelihood of the data as a function of $b$ and the inverse covariance matrix $A =(\s\s^T)^{-1}$, which is the joint density of $\BX_{t_1:t_m}$, is given by 
\begin{align}\label{def-lhood-0}
L(b,A | \BX_{t_1:t_m}) = \prod_{i=1}^m p_\Delta(X(t_i)|X(t_{i-1})), \quad t_0 =0,\ X(0) = x_0.
\end{align}
The natural loss function here is the negative log likelihood, $-\ln L $, and the functions $b$ and $A =(\s\s^T)^{-1}$ are learned through minimizing it over an RKHS, subject to a penalty term. Now  the transition densities $p_s(\cdot|\cdot)$ are usually not available in closed form, and in practice, we often work with a discretized version of the SDE \eqref{eq:sde0}. In this paper we will consider the Euler-Maruyama approximation of \eqref{eq:sde0}  given by
\begin{align}\label{eq:EM}
	X(t_i) = X(t_{i-1}) +b(X(t_{i-1})\Delta+ \s(X(t_{i-1})) (W(t_i) - W(t_{i-1})), \quad \Delta = t_i - t_{i-1} \ll 1
\end{align}
which has a weak-error of order 1, same as the Milstein-scheme \cite{GrTa13}.
The advantage of Euler-Maruyama (EM) approximation, is that the transition density of the discretized chain \eqref{eq:EM}, which can be thought of as an approximation to that of the original process $X$, is simply given by
\begin{align*}
	p^{EM}_{\Delta}(x'|x) = \No_d(x'|x+b(x)\Delta, \s\s^T(x)\Delta).
\end{align*}
 Consequently, the likelihood function $L$ in \eqref{def-lhood-0} will be approximated by $ L^{EM}$, the likelihood function of the EM chain \eqref{eq:EM}, which is defined in a way similar to \eqref{def-lhood-0} with  the approximate transition densities $ p^{EM}_\Delta(X(t_i)|X(t_{i-1}))$ replacing the exact $p_\Delta(X(t_i)|X(t_{i-1}))$. Discretized chains corresponding to Milstein-scheme or higher-order approximations like  Runge-Kutta type schemes do not have such simple closed forms of transition densities and are comparatively difficult  to work with for development of learning algorithms.


Since our objective is to learn vector-valued functions, the corresponding minimization problem needs to be cast in RKHS corresponding to matrix-valued kernels (see Definition \ref{def:matRK}). Let $\knl_0:\R^{d} \times \R^d \rt \R^{d\times d}$ and $ \knl_1:\R^{d}\times \R^{d} \rt \R^{d^2\times d^2}$ be reproducing kernels with associated RKHS $\Hsp_0$ and $\Hsp_1$. Let  $\Hsp = \Hsp_0\oplus_e\Hsp_1$, and $J:[0,\infty) \rt [0,\infty]$  a strictly increasing function. Then Corollary \ref{cor:rep} gives  the following result.

\begin{theorem}\label{th-rep}
Consider the following minimization problem
	\begin{align*}
	\min_{(b,A) \in \SC{H}} -\ln L^{EM}(b,A | X(t_1),X(t_2),\hdots, X(t_m))+J(\|(b,\ve_{d\times d}(A))\|).
	\end{align*}
	where $A = (\s\s^T)^{-1}$.
	Then there exists a solution  to the above minimization problem and every minimizer $(b^*,A^*)$ is of the form
	\begin{equation} \label{eq:bs-fin}
	b^*(\cdot) = \sum_{i=1}^m  \knl_0(\cdot,X(t_i)) \wtb^*_i,\quad \ve_{d\times d}(A)(\cdot)) = \sum_{i=1}^m  \knl_1(\cdot,X(t_i)) \alpha^*_i \qquad \wtb_i \in \R^d,\quad \alpha_i \in \R^{d^2}\\
	\end{equation}
\end{theorem}	
Here $\ve_{d\times d}(M)$ is vectorization of a $d\times d$ matrix $M$.  The next part of the paper focuses on estimating the weight coefficients in the summations in \eqref{eq:bs-fin}.


\subsection*{Computational aspects}
The  computational part of the paper focuses only on the nonparametric learning of  the {\em drift coefficient} $b$ from high-frequency data. More specifically, we consider It\^o diffusion with unknown drift function $b$ but whose   diffusion coefficient has the parametric form $\s(x) =  \sigma_0(x) \vas$, with a {\em known} function $\s_0: \R^d \rt \R^{d\times d}$ and an {\em unknown} $d\times d$ parameter matrix $\vas$.
The  transition density of the discretized chain \eqref{eq:EM} in this case is given by 
\begin{align}\label{eq:em-dens}
p^{EM}_\Delta(x'|x) = \No_d(x'|x+b(x)\Delta, \s_0(x)\vas\vas^T\s^T_0(x)\Delta), 
\end{align}
The assumption of parametric form of the diffusion coefficient is made for certain computational advantages. The case where both $b$ and $\s$ are unknown functions requires significantly different techniques and is the subject of our future work. We however do note that the framework in this paper covers the important class of SDEs with constant diffusion coefficients. \\
%

\np
{\bf Estimating the minimizer:} Now there are two approaches to estimate the minimizer $b^*$, or equivalently, $\bm{\wtb}^* \equiv (\wtb^*_1,\wtb^*_2,\hdots,\wtb^*_m)$. The first obvious way is to solve the optimization problem either by an optimization algorithm (e.g. stochastic gradient descent) or in closed form when it is possible (e.g. in the case, the penalty function $J(u) =\|u\|^2$). This gives a point-estimate of $\bm{\wtb}^*$, a main drawback of which, as already pointed out by Tipping  \cite{Tipp01} in the regression case, is the absence of a reliable measure of uncertainty. Any ad-hoc post processing of the estimate to get some quantification of the uncertainty is artificial due to lack of probabilistic framework and often leads to unreliable results.

A natural remedy to the above problem is  a Bayesian approach, which is the focus of this paper. This entails assigning a prior distribution $p_{prior}(\cdot)$ on the weight vector $\bm{\wtb} = (\wtb_1,\wtb_2,\hdots,\wtb_m)$, and estimating the posterior distribution, $p_{post}(\bm{\wtb}|\BX_{t_1:t_m})$. Justifying the finite expansion, $b(\cdot) = \sum_{i=1}^m  \knl_0(\cdot,X(t_i)) \wtb_i$ as an ``ideal form" of the drift function $b$ by Theorem \ref{th-rep}, the posterior distribution, $p_{post}(\bm{\wtb}|\BX_{t_1:t_m})$, efficiently captures the {\em uncertainty} in our estimator in the $m$-dimensional parameter space.
  The connection between the optimization problem and the Bayesian approach, as has been described numerous times in the literature in other contexts (e.g. see \cite{MacK92}), is  the observation that the negative of the cost function in Theorem \ref{th-rep} (seen as a function of $\bm{\wtb}$) is the log posterior-density of $\bm{\wtb}$ under the prior $p_{prior}(\bm{\wtb}) \propto \exp\{-J(\bm{\wtb})\}$, where by a slight abuse of notation, we denote $J(\bm{\wtb}) = J(\bm{\wtb}^T \bm{\SC{K}}_0 \bm{\wtb}) = J(\|b\|^2)$ with $b(\cdot) = \sum_{i=1}^m \knl_0(\cdot, X(t_i))\wtb_i$. Here, $\bm{\SC{K}}_0 =((\knl_0(X(t_i), X(t_j))))$ is the Gram matrix associated with the kernel $\knl_0$. Thus $\bm{\wtb}^*$, the solution of the penalized optimization problem, is interpreted as a-posterior mode (MAP) of the posterior distribution of $\bm{\wtb}$. Importantly, this observation shows that Bayesian approach allows one to use a much larger class of priors on $\bm{\wtb}$ than the class of penalty functions to achieve desired objectives like sparsity; in particular, one can now use priors which do not have closed form expressions.


\subsection{Sparsity and Shrinkage priors}
Since for the SDE model, the RKHS framework requires that the number of terms in the finite expansion of $b$ equals the number of data-points, $m$,  getting a sparse estimate of $\{\wtb_i:i=1,2,\hdots,m\}$ is necessary. This would not only lead to reduction in complexity but will protect us from an over-parametrized model. But it is important to understand why a sparse solution is expected in this case. Note that  shrinkage priors  in the context of SDEs hold an appeal that is interestingly different from that in  usual regression setups. Here our ``predictors" come in the form of correlated data. An efficient algorithm should not ideally  place non-zero weights on all data-points that are very close to each other. Data points clustered together in a small region of the data space, will not provide  information individually over and above what could be provided by  few representative points of the cluster. Such clusters can be  typically formed by slow movement of SDE resulting in two successive data-points, $X(t_i)$ and $X(t_{i+1})$,  differing only by a  little margin. It could also be formed by  multiple visits of the SDE trajectory to the same regions of the  data  space due to positive recurrence or ergodicity of the system. In other words, the presence of both $\knl_0(\cdot, X(t_i))$ and $\knl_0(\cdot, X(t_{j})$ is unnecessary in the finite-expansion of $b$ when $X(t_i)$ and $X(t_j)$ are nearly identical, and only a subset of $\{\knl_0(\cdot, X(t_i))\}$ is relevant for learning $b$. In fact, this shows why we expect the methodology of the paper to work for SDEs which are positive recurrent (ergodic). It guarantees that we have enough data points to learn about the relevant weights $\wtb_i$, which might not be true for other types of SDEs.

In the optimization framework, sparsity can be induced by different cost functions $J$ in the minimization problem
\begin{align*}
 \min_{\bm{\wtb}}\lf[-\ln L^{EM}(b,A | X(t_1),X(t_2),\hdots, X(t_m))+J(\bm{\wtb})\ri]
\end{align*}
with $b(\cdot) = \sum_{i=1}^m \knl_0(\cdot, X(t_i))\wtb_i.$ While $l_2$ cost function often does not result in noticeable sparsity, other choices of $J$, for example, the lasso penalty of Tibshirani \cite{Tibs96} results in certain $\wtb_i$'s becoming zero. Within the Bayesian framework, popular choices of shrinkage prior $p_{prior}(\bm{\wtb})$ lie in the normal scale-mixture family which in particular include $\tdst$-prior \cite{Tipp01}, double-exponential \cite{PaCa08} and Horseshoe priors \cite{CarPol09, CPS10}. A survey of some of the popular shrinkage priors used for penalized regression problems can be found in \cite{VOM19} (also see the references therein). The MAP estimate corresponding to double-exponential prior of course is the same as the lasso estimate, but the posterior mode often lacks nice theoretical properties and is also unsuitable from Bayesian perspective.
In fact, a Bayesian approach which touts model averaging does not expect model-averaged weights to be exactly zero! It is more reasonable to consider a weaker-form of sparsity which aims to decrease $\|\bm{\wtb}\|$ for some suitable norm --- resulting in shrinkage rather than selection of the weights. 

In the Bayesian framework,  an established method  inducing shrinkage  is  by choosing appropriate heavy-tailed distributions with sharp peak at $0$ as shrinkage priors. While the sharp peak results in shrinkage of most of the coefficients, the heaviness of the tail allows truly relevant weights  to  shift away from  $0$.  The use of shrinkage priors is a first, to our knowledge, in the context of SDE models.

In this paper we use two types of priors on $\{\wtb_i\}$ to induce sparsity -  $\tdst$-distributions and the Horseshoe distribution. Since the weights $\wtb_i$ are vector valued, it should be noted that multidimensional versions of the above prior distributions need to be used. While multidimensional $\tdst$-distribution is standard in the literature, such is not the case for Horseshoe.  We describe a natural and easy-to-implement adaptation of  the classical Horseshoe to $d$-dimension later in the section.\\


\np
{\bf $\tdst$-prior:} To induce sparsity we assume  multivariate  $\tdst_{d}(\cdot|\nu,0, U)$ prior with $\nu$ degrees of freedom on each $\wtb_i$. Recall that the multivariate  $\tdst_{d}(\cdot|\nu, \mu, U)$ density function is given by
\begin{align}\label{eq:den-t}
\tdst_{d}(x|\nu, \mu, V) = \f{\G\lf((\nu+d)/2\ri)}{\G(\nu/2)\det(U)^{1/2}(\nu \pi)^{d/2}} \lf[1+\f{1}{\nu} (x -\mu)^TV^{-1}(x-\mu))\ri]^{-\f{\nu+d}{2}}, \quad x \in \R^d.
\end{align}
Now $\tdst_{d}(\cdot|\nu, \mu, V)$ can be written as a normal scale mixture with covariance matrix mixed with inverse Wishart distribution; More specifically,
\begin{align*}
\mrm{t}_{d}(x|\nu, \mu, U) = \int _{\R^{d\times d}} \No_d(x|\mu, \loc) \SC{IW}_d(\locm|\nu+d-1, U) d\locm
 \end{align*}
 This facilitates Gibbs sampling of the posterior by the standard technique of  augmentation of the parameter space. To complete the Bayesian framework, we also need to assume prior on the starting data point $X(t_1)$. We assume 
 $ X(t_1)|x_0, \{\wtb_i\}, \vas\vas^t \sim \No_d(\cdot|x_0+ b(x_0)\Delta, \s_0(x)\vas(\s_0(x)\vas)^T\Delta)$ 
 with hyperparameter $x_0$. Notice that this prior is consistent with the dynamics of $X$ (c.f \eqref{eq:EM} and \eqref{eq:em-dens}) and can be interpreted as follows: designating $t_1-\Delta$ as the starting time, $t=0$, we assume that  $X(0) = x_0$, where we choose $x_0$ to be close to the first observation $X(t_1)$. Ideally, we should assign a proper prior to $X(0)$, for example, a uniform prior on a small ball around $X(t_1)$, but simply fixing $X(0) = x_0$ near the data-point $X(t_1)$ (or equivalently, assigning Dirac $\delta_{x_0}$ prior to $X(0)$) does not affect the performance of the algorithms. Our Bayesian hierarchical framework is described below.\\

%

\begin{tcolorbox}
	\underline{\bf Bayesian hierarchical framework I: $\tdst$-prior}
	
	\begin{itemize}
	       \item $X(t_1-\Delta) \equiv X(0) \sim \delta_{x_0}$.
		\item $\{\BX_{t_1:t_m} =(X(t_1), X(t_2), \hdots, X(t_m))) \Big | \bm{\wtb}, \vas, x_0\Big\}$ governed by the transition probabilities \eqref{eq:em-dens}, which is the result of Euler-Maruyama approximation, \eqref{eq:EM}
		
		\item 
		Mean-zero Gaussian prior on the parameter $\bm{\wtb}$: for $i=1,2,3,\hdots,m$, $\wtb_i\stackrel{iid}\sim \No_d(0,\locm_i),$ where each $\eta_i$ is a $d\times d$ positive definite matrix.
		
		\item Inverse Wishart prior on the hyperparameter $\locm_i$: $\locm_i \sim \SC{IW}_d(\nu+d-1,U)$ for $i=1,2,\hdots,m$.
		
		\item Inverse Wishart prior on the parameter $\vas\vas^T$: $\vas\vas^T \sim \SC{IW}_d(n,V)$.
	\end{itemize}	
\end{tcolorbox}

The $\locm_i$ controls the strength of the coefficients $\wtb_i$ and therefore the relevance of the data-point $X(t_i)$. 
For one-dimensional SDEs, this is of course equivalent to putting an inverse-gamma prior on the variance of the zero-mean normal distributions of $\wtb_i$.

For multidimensional SDEs, an alternate simpler $\tdst$-like prior can also be assigned to $\wtb_i$ by setting $\locm_i = \loc_i I_{d}$ with 1-dimensional inverse gamma prior on the scalar $\loc_i$. The main advantage of the simpler prior is that it requires much less number of hyperparameters than the multi-dimensional $\tdst_d$-prior resulting in potential savings in computational complexity. 

With the above priors, the conditional distribution of each of the parameters given the rest have closed forms and can be deduced from Lemma \ref{lem:cond-0}. This results in the following Gibb's algorithm for (approximately) generating $\bm{\wtb}, \{\locm_i\} $ and $\vas\vas^T$ from the posterior distribution  $p_{post}(\bm{\wtb}, \{\locm_i\}, \vas\vas^T|\BX_{t_1:t_m} )$. 

\vs{.2cm}
\begin{algorithm}[H] 
	\DontPrintSemicolon
	\KwIn{The data $\BX_{t_1:t_m} =(X(t_1), X(t_2), \hdots, X(t_m))$, $x_0$
		discretization step \(\Delta\), number of iterations $L$. }
	\KwOut{ $\bm{\wtb}, \vas\vas^T, \{\locm_i\}$ from the posterior density. }
	\While{  \(l<L\)}  {
		Generate  $(\bm{\wtb}|\BX_{t_1:t_m}, \vas\vas^T, \{\locm_i\}) \sim \No_d(\cdot|\bmn, \bcm) $ where $\bmn$ and $\bcm$ are defined by \eqref{eq:b-post}.\;
		
		Generate $(\vas\vas^T)| \BX_{t_1:t_m},\bm{\wtb}, \{\locm_i\} \sim \SC{IW}_d(n+m, V_{post})$, where 
		$V_{post}$ is defined by \eqref{eq:sigma-post}.\;
		Generate $\locm_i| \BX_{t_1:t_m}, \bm{\wtb},\vas\vas^T \sim \SC{IW}_d(\nu+d, U^{-1}+\beta_i\beta_i^T), \quad i=1,2,\hdots,m$ independently.\;
		$l =l+1$
	}
	\caption{ Gibb's algorithm for high frequency data. }
	\label{algo:Gibbs-t}
\end{algorithm}
%

%
%
%
%

\vs{.2in}
\np
{\bf Horseshoe type prior:}
We next employ a global-local class of priors from the normal scale-mixture family which has potentially better shrinkage characteristics than the $\tdst$-prior \cite{PoSc11}. In our context of $d$-dimensional $\wtb_j$, this is described by
\begin{align*}
\wtb_j| \locm_j, \glom \sim \No(0,\locm_j\glom), \quad \locm_j \sim p_{prior}(\locm_j), \quad \glom \sim p_{prior}(\glom).
\end{align*}
$\glom$, which denotes the global variance component, is akin to the regularization parameter in the penalized optimization problem and its purpose is to attempt to shrink all the weights $\{\wtb_j\}$. This requires $\glom$ to be small in some appropriate sense.  The local variance component $\{\locm_j\}$ should be such that it can relax the shrinkage effect for those coefficients whose magnitude is large. Now
 it is easy to see from Lemma \ref{lem:cond-0} that
\begin{align*}
\EE\lf[\bm{\wtb}| \BX_{t_1:t_m}, \{\locm_j\}, \glom\ri] =&\ (I-S)\hat {\bm{\wtb}}_{MLE}
\end{align*}
where the shrinkage factor, $S = (I+\vas^{-2}\eta\SC{K}_0^T\SC{K}_0)^{-1}$ with $\eta = diag(\locm_1, \locm_2,\hdots,\locm_m)\ot \glom$, and $\hat {\bm{\wtb}}_{MLE} = \Delta (\SC{K}_0^T\SC{K}_0)^{-1}\SC{K}_0\vartheta$ is the standard MLE estimate of $\bm{\wtb}$ based on the likelihood, $ L^{EM}(b|\BX_{t_1:t_m})$. Here we assumed for simplicity that the diffusion coefficient, $\s(x) \equiv \vas I, \ \vas \in \R$.
This points to the necessary characteristics of the prior distributions of the hyperparameters, $\glom$ and $\locm_j$: (a) $p_{prior}(\glom)$ should have a sharp peak at $0$, and (b)  $p_{prior}(\locm_j)$ should have heavy tails.

Instead of choosing $d\times d$- dimensional probability distributions as priors for $\locm_j$ and $\glom$ we set $\locm_j = \loc_j I_{d\times d}$ and $\glom = \glo I_{d\times d}$ with one-dimensional priors on $\loc_j$ and $\glo$ satisfying the above criteria. These choices of priors require a much smaller number of hyperparameters, leading to potentially significant savings in computational complexity while allowing an easier-to-implement extension of 1-D global-local priors for multidimensional parameters.

If $\glo = 1$, then an inverse gamma-prior on $\loc_j$ leads to (a multidimensional version of) $\tdst$-prior on the $\wtb_j$. Although the inverse-gamma is popular as a choice of mixing distribution for the variance components of normal-scale mixture family of priors, it can be informative in certain cases leading to non-robust estimation of the $\wtb_j$. Moreover, for an inverse-gamma distribution, $p_{prior}(\glo) \rt 0$ as $\glo \rt 0$. Now note that $p(\BX_{t_1:t_m}|\glo)$ does not converge to $0$, when $\glo \rt 0$; Consequently, $p(\tau|\BX_{t_1:t_m}) \stackrel{\glo \rt 0} \rt 0$ forcing the posterior distribution of $\tau$ to biased away from $0$, and thereby assigning low probability to that part of the parameter space where benefits of shrinkage is desired! This has already been pointed out for regression problems by Gelman \cite{Gel06} (also see Polson and Scott \cite{PoSc12}) and is also true for data from dynamical systems that are of interest in this paper. These issues with the inverse-gamma prior can be mitigated by averaging its scale parameter with another appropriate distribution, e.g. a gamma distribution; see \cite{PePeRa17}. This leads to a scaled $\Fdst$-distribution. The density $\Fdst(\cdot|\nu_1,\nu_2,c)$ of $\Fdst$-distribution (or Beta-prime $(2\nu_1, 2\nu_2)$ distribution with scaling parameter $c$) with degrees of freedom $\nu_1, \nu_2$ and scaling parameter $c$ is given by

\begin{align}\label{eq:den-F}
\begin{aligned}
\Fdst(z|\nu_1, \nu_2, c) =&\  \f{\G\lf(\f{\nu_1+\nu_2}{2}\ri)}{\G\lf(\f{\nu_2}{2}\ri)\G\lf(\f{\nu_1}{2}\ri) c^{\nu_1/2}} z^{\f{\nu_1}{2}-1} (1+z/c)^{-\f{\nu_1+\nu_2}{2}}\\
=&\ \int \SC{IG}(z|\nu_2/2,\theta)\ \SC{G}(\theta|\nu_1/2, c^{-1}) \ d\theta.
\end{aligned}
\end{align}
Elementary formal calculations show that 
\begin{align*}
\Fdst(z|\nu_1, \nu_2, c) \stackrel{z \rt 0}\approx z^{\nu_1/2-1}, \quad \Fdst(z|\nu_1, \nu_2, c) \stackrel{z \rt \infty}\approx z^{-{\nu_2/2-1}}
\end{align*}
which in turn show that the first degree of freedom, $\nu_1$, controls 
the behavior of $\Fdst$-distribution around zero, while the behavior in tails is controlled by the second degree of freedom, $\nu_2$. Choosing a smaller value of $\nu_1 < 2$ will result in a pole at $0$, and smaller values of $\nu_2$ will lead to heavier tails. Formal calculations also indicate that
\begin{align*}
p_{prior}(\wtb_j =0|\loc_j) =&\ \int \No(\wtb_j=0|\loc_j\glo I)\Fdst(\glo|\nu_1, \nu_2, c) d\glo\\
\propto&\ \int (\glo)^{(\nu_1-d)/2 - 1} (1+\glo/c)^{-(\nu_1+\nu_2)/2} d\glo
= \infty
\end{align*}
if $\nu_1\leq d$, as the last integral then is proportional to integral of an improper $\Fdst$-density.  $\Fdst(\nu_1=\nu_2=1,c=1)$-prior on $\loc_j$ and $\glo$ (or equivalently, Half-Cauchy(0,1)-prior on $\l_j$ and $\tau$) leads to the Horseshoe prior (or more precisely, a multidimensional version of it) on $\wtb_j$. However, in the case of correlated temporal data from dynamical systems,  these default choices of $\nu_1=\nu_2=1, \ c=1$,  can result in $\glo$ to be near-zero value shrinking all the weights $\wtb_j$ substantially. It might be necessary to adjust the degrees of freedom parameters to counter such strong shrinking force - for example, by using a $\Fdst$-prior on $\loc_j$ having heavier tails (that is, lower value of second degree of freedom,  $\nu_2$) than $\Fdst(\nu_1=\nu_2=1,c=1)$ to recover the relevant weights.

The Bayesian hierarchical framework with the above choices is summarized below.

\begin{tcolorbox}
	\underline{\bf Bayesian hierarchical framework II: Horseshoe-type priors}
	
	\begin{itemize}
		\item $X(t_1-\Delta) \equiv X(0) \sim \delta_{x_0}$.
		\item $\{\BX_{t_1:t_m} \Big | \bm{\wtb}, \alpha\Big\}$ governed by the transition probabilities \eqref{eq:em-dens}, which is the result of Euler-Maruyama approximation, \eqref{eq:EM}
		
		\item Independent mean-zero Gaussian priors on the parameters $\wtb_i$: $\wtb_i \sim \No_d(\cdot|0,\loc_i\glo I_d),$ where   $\loc_i, \glo \in [0,\infty)$.
		
		\item Inverse Gamma priors on the hyperparameters $\loc_i$ and $\glo$: $\loc_i \sim \SC{IG}(\cdot|\hyp_i,\theta_i)$ for $i=1,2,\hdots,m$ and $\glo \sim \SC{IG}(\cdot|\hyp^0,\theta^0)$
		
		\item Gamma priors on the hyperparameters $\theta^0, \theta_i$: $\theta_i  \sim \SC{G}(\cdot|\hpa,\hpb)$,  for $i=1,2,\hdots,m$, and 
		      $\theta^0  \sim \SC{G}(\hpa^0,\hpb^0)$
		
		\item Inverse Wishart prior on the hyperparameters $\vas\vas^T$: $\vas\vas^T  \sim \SC{IW}_d(\cdot|n,V)$.

	\end{itemize}	
\end{tcolorbox}

Equation \ref{eq:den-F} leads to easy sampling of the parameters from the posterior distribution via Gibbs sampling. This is summarized in the algorithm below, and the computational details  are given in Lemma \ref{lem:cond-0} in the Appendix. The following notations are convenient for descriptions of Algorithm \ref{algo:Gibbs-HS} and Lemma \ref{lem:cond-0}.\\

\np
{\em Notation:} Let $\salg$ denote the $\s$-field generated by $\BX_{t_1:t_m} =(X(t_1), X(t_2), \hdots, X(t_m)))$, and the parameters $\bm{\wtb}, \vas\vas^T, \{\loc_i:i=1,2\hdots,m\}, \glo, \{\theta_i: i=1, 2, \hdots, m\}, \theta^0$ (viewed as random variables on the same probability space.). Let $\salg_{-\bm{\wtb}}$ be the $\s$-field generated by the above random elements except $\bm{\wtb}$, $ \salg_{-\{\loc_i\}}$ the $\s$-field generated by the above random elements except $\{\loc_i: i=1,2,\hdots,m\}$. The $\s$-fields $\salg_{-\vas\vas^T}, \salg_{-\{\theta_i\}}, \salg_{-\theta^0,\{\theta_i\}}$, etc are defined similarly.

\begin{algorithm}[H]
	\DontPrintSemicolon
	\KwIn{The data $\BX_{t_1:t_m}$, $x_0$, 
		discretization step \(\Delta\), number of iterations $L$. }
	\KwOut{ $\bm{\wtb}, {\loc_j}, \glo, \vas\vas^T$ from the posterior density. }

	\While{  \(l<L\)}  {
		Generate  $\bm{\wtb}|\salg_{-\bm{\wtb}} \sim \No_d(\cdot|\bmn, \bcm) $ where $\bmn$ and $\bcm$ are defined by \eqref{eq:b-post}.\;
		Generate $\vas\vas^T| \salg_{-\vas\vas^T} \sim \SC{IW}_d(n+m, V_{post})$, where 
		$V_{post}$ is defined by \eqref{eq:sigma-post}.\;
		Generate $\loc_k| \salg_{-\{\loc_i\}}\ \sim\ \SC{IG}\lf(\cdot \Big|(d+2\hyp_k)/2,\  \f{1}{2}\wtb^T_k\wtb_k/\glo+\theta_k\ri), \quad k=1,2,\hdots,m$ independently.\;
		Generate $\glo|  \salg_{-\glo} \ \sim\ \SC{IG}\lf(\cdot \Big| (md+2\hyp^0)/2,\  \theta^0+\f{1}{2}\sum_{k=1}^m \wtb_k^T\wtb_k/\loc_k\ri)$.\;
		Generate $\{\theta_k\}$ and $\theta^0$ as $\theta_k| \salg_{-\theta^0,\{\theta_i\}}\ \sim\ \SC{G}\lf(\cdot|\hyp_k+\hpa,\hpb+1/\loc_k\ri)$, and $\theta| \salg_{-\theta^0,\{\theta_i\}}\ \sim\ \SC{G}\lf(\cdot|\hyp^0+\hpa^0,\hpb^0+1/\glo\ri).
		$\;
		$l =l+1$
	} 
	
	\caption{ Gibb's algorithm for high frequency data with Horseshoe prior. }
	\label{algo:Gibbs-HS}
\end{algorithm}

An alternate option would have been to impose independent one-dimensional Horseshoe prior on each component $\wtb_{jl}, l=1,2,\hdots,d;\ j=1,2,\hdots, m.$ A version of this prior has previously been used by one of the authors for a multi-outcome regression model \cite{KuMiGa21}. There the local shrinkage effects, while varying among individual predictor values, were shared across multiple dimensions of the same predictor, and the global component varied across different dimensions. While these types of priors  may be more natural for the  multi-outcome regression models of \cite{KuMiGa21} to  allow more intra-dimensional variability, their use in the context of multidimensional dynamical systems lacks strong justification.  
Rather the  significantly higher number of additional hyperparameters that these priors require will lead to substantial increase in the complexity and run-time of the resulting Gibb's algorithm.

\section{Simulation Results} \label{sec:sim}

We next demonstrate the effectiveness of our algorithm for four SDE models. The SDEs considered are ergodic with a unique stationary distribution. As mentioned earlier, this is exactly the class of models where we expect our algorithms to work best. Ergodicity will ensure that the SDE will visit the relevant states multiple times. This will lead to a sufficient number of data points corresponding to each such states over a finite-time interval which in turn will result in more accurate learning of the drift function $b$.

From a discrete path from each of the SDE models, we use our algorithms to generate samples of $\bm{\wtb}$ from the posterior distribution. The (posterior) mean of these $\bm{\wtb}$-samples gives the estimated function $\hat b$, which is plotted against the true $b$. The corresponding mean square error (MSE) is also reported. While closeness between $\hat b$ and the true $b$ demonstrates the effectiveness of our learning algorithms, a further validation of the algorithm comes from matching the equilibrium (or the stationary) distribution of the estimated SDE with that of the true one. This shows that the behavior of the estimated SDE matches with that of the true SDE at future times --- further beyond the time-range of the observed data. This is important as it demonstrates the predictive power of the learned SDE model and shows that the closeness between the true and the estimated drift functions, $b$ and $\hat b$, is indeed due to the accuracy of the algorithms and not due to overfitting. The latter despite giving good fit within the time-range of the data would often result in markedly different behaviors of the paths of the corresponding SDEs at unobserved future times. The closeness between the two stationary distributions is  assessed  through the  Kolmogorov metric, $\sup_{x}|F_{st}(x) - \hat F_{st}(x)|$, where $F_{st}$ and $\hat F_{st}$ respectively denote the cumulative distribution functions (CDFs) of the stationary distributions of the true and the estimated SDEs. Specifically, the former refers to the SDE driven by the  true drift function $b$ and the  diffusion parameter $\vas^2$ while the latter corresponds to the SDE driven by their estimated versions $\hat b$ and $\hat \vas^2$. 

We used Gaussian kernels for our simulation studies. Specifically, for the 1-D models, we used the kernel $\kappa_0(x,y) = \exp(-(x-y)^2/2)$ and for the multidimensional Michaelis-Menten kinetics in Model 3, we used $\knl_0 = \kappa_0 I_{3}$.\\

\np
{\bf Model 1: Double-well potential SDE} \\
Our first model is an overdamped Langevin SDE representing the motion of a particle in a double-well potential given by $u(x) = x^4-2x^2$. The trajectory of the particle depends on two factors: a (deterministic) driving force $b(x) = -u'(x) = 4(x-x^3)$, and random perturbations modeled by an additive Brownian noise. The potential has two wells (minimum energy states) located at $\pm 1$, and the driving random noise occasionally makes the particle transition from one minima to the other. The dynamics of the particle is thus highly non-linear and the corresponding SDE given by
\begin{align*}
dX(t) = 4X(t)(1-X^2(t)) + \vas dW(t).
\end{align*}
Such SDEs are also important in mathematical finance. The two wells lead to a bimodal stationary distribution whose density is given by
\begin{align*}
\pi_{st}(x) \propto \exp\lf(\f{2x^2-x^4}{2\vas^2}\ri).
\end{align*}
Our data points come from the above SDE with $\vas=1$, and we use Algorithm \ref{algo:Gibbs-t} and Algorithm \ref{algo:Gibbs-HS} to estimate the entire drift function $b$, and the diffusion parameter $\vas$. For this we use a scaled $\tdst(\cdot|\nu=2, c=1, \mu=0)$-prior on the weights $\wtb_k$ (that is, $\wtb_k \sim \No(\cdot|0,\loc_k), \loc_k \sim \SC{IG}(1,2)$) in Algorithm \ref{algo:Gibbs-t} (with inverse-gamma replacing inverse-Wishart), and we use the parameters $\hyp_i = \hyp^0 = \hpa= \hpa^0= 1/2, \hpb=\hpb^0=1$ (that is, classical HS prior) for Algorithm \ref{algo:Gibbs-HS}. For both the algorithms we use $\SC{IG}(1,2)$-prior on the diffusion-parameter $\vas^2$. Figure \ref{fig:dw} gives a visual representation of the performances of the algorithms: (\textbf{a}) plots the real drift  function $b$ and the estimated $\hat b$ in three cases -  with no-shrinkage, shrinkage with $\tdst$ and HS priors on the weights; (\textbf{b}) plots a histogram of the weights $\wtb_k$, which shows the effect of shrinkage priors; (\textbf{c}) compares the stationary distributions of the SDE with estimated drift function $\hat b$ in three cases (no-shrinkage, $\tdst$ and HS shrinkage priors) with the true stationary distribution of the double-well potential SDE; (\textbf{d}) shows the corresponding P-P plots.





\begin{figure}[h]
\centering
\includegraphics[width=0.8\textwidth]{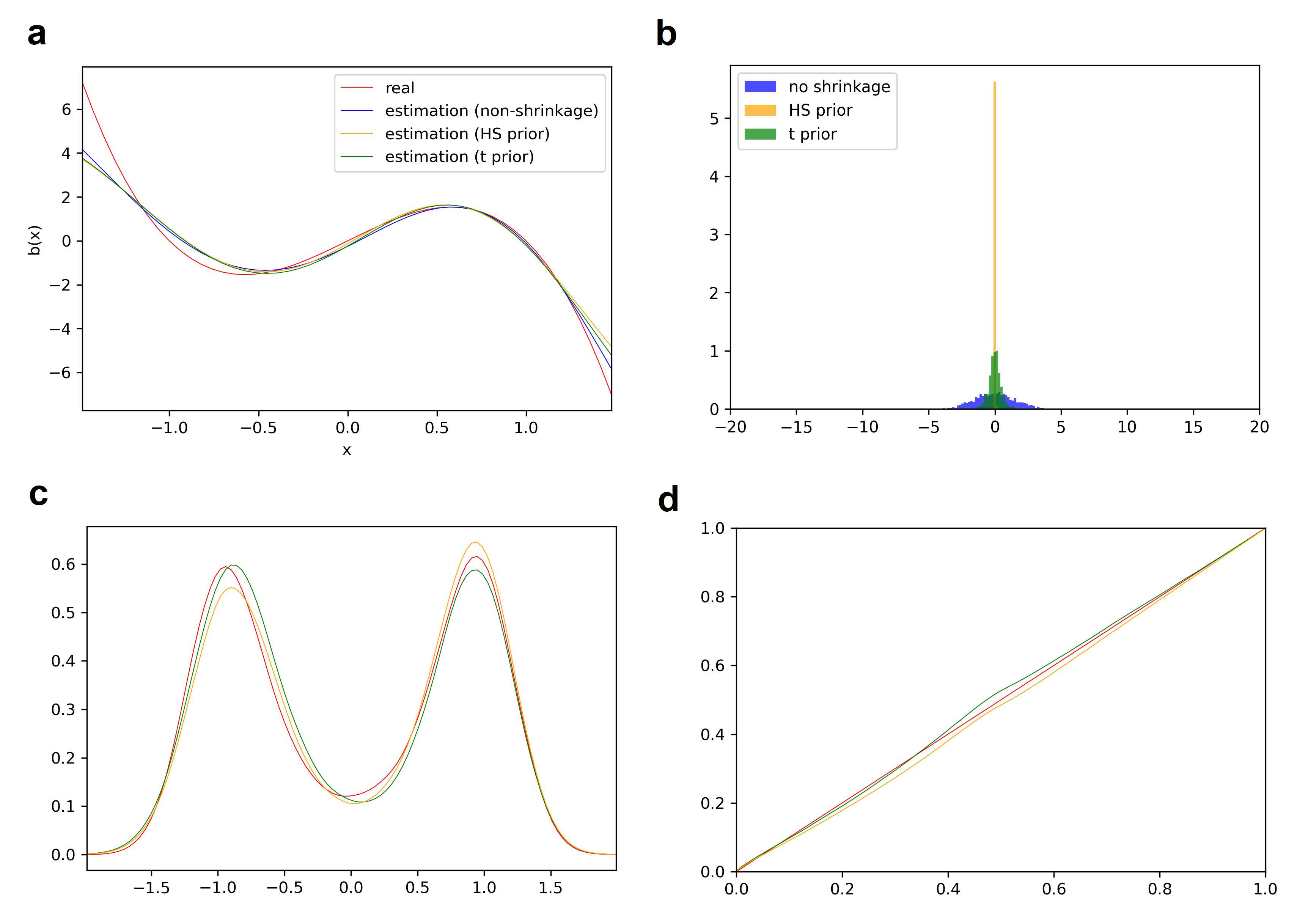}
\caption{Double-well potential SDE. \textbf{a}: comparison of estimated function $\hat b$ with true $b$. \textbf{b}: histogram of $\beta_i$'s. \textbf{c}: comparison of the stationary distributions of the SDE driven by estimated $\hat b$ and true $b$. \textbf{d}: PP-plots of the stationary distributions of the estimated SDE against that of the original SDE.}
\label{fig:dw}
\end{figure}

Figure \ref{fig:dw}-(b) is noteworthy as it shows that both $\tdst$ and HS priors were successful in giving sparse solutions for the weights, $\wtb_i$, with HS prior producing significantly higher degree of sparsity compared to $\tdst$-prior as evidenced from much sharper peak of the histogram near $0$. At the same time other figures show that both the shrinkage priors lead to almost identical $\hat b$ matching the accuracy of the estimate without shrinkage. The MSE and the Kolmogorov metric values in all the cases are in the range 0.27-0.29 and 0.7 - 0.8, respectively.

Better accuracy is expected with more data, which can be either because of higher frequency of observations (that is, lower value of $\Delta$) or more observations over longer time range $[0,T]$. 

\begin{table}[H]
\centering
\begin{tabular}{|p{2.4cm}||p{1.2cm}|p{1.2cm}|p{1.2cm}|p{1.2cm}|p{1.2cm}|p{1.2cm}|} 
 \hline
 $T$ & 40 & 40 & 40 & 80 & 60 & 20 \\ 
 $\Delta$ & 0.025 & 0.05 & 0.1 & 0.05 & 0.05 & 0.05 \\ [0.5ex] 
 \hline\hline
 $\tdst$-prior & 0.3035 &	0.3966 & 0.7234 & 0.2818 & 0.2971 &	0.5128  \\ 
 HS-prior & 0.3258 &	0.4193 &	0.9442 &	0.2890 &	0.3362 & 0.7106 \\[1ex] 
 \hline
\end{tabular}
\caption{MSE of $\hat b$ for $\tdst$ and HS priors.}
\label{DW-table:1}
\end{table}

\begin{table}[H]
\centering
\begin{tabular}{|p{2.4cm}||p{1.2cm}|p{1.2cm}|p{1.2cm}|p{1.2cm}|p{1.2cm}|p{1.2cm}|} 
 \hline
 $T$ & 40 & 40 & 40 & 80 & 60 & 20 \\ 
 $\Delta$ & 0.025 & 0.05 & 0.1 & 0.05 & 0.05 & 0.05 \\ [0.5ex] 
 \hline\hline
 $\tdst$-prior & 0.1494 & 0.2102 & 0.2047 & 0.0707 & 0.0627 & 0.2702  \\ 
 HS-prior & 0.1325 & 0.2047 & 0.2486 & 0.0817 & 0.0913 & 0.2869 \\[1ex] 
 \hline
\end{tabular}
\caption{Kolmogorov metric between the CDFs of the stationary distributions of the estimated and true SDEs.}
\label{DW-table:2}
\end{table}

 This is corroborated by Table \ref{DW-table:1} and Table \ref{DW-table:2}, which list the values of MSE and the Kolmogorov metric in two cases -  (i)  fixed observation-time range $[0,T]$, but increasing $\Delta$, and (ii) fixed observation frequency $\Delta$ but increasing time range $[0,T]$. \\

\np
{\bf Model 2: Variant of Double-well potential SDE}

Our second model is a variant of the above double-well potential SDE with a multiplicative noise structure. The specific equation is given by
\begin{align*}
dX(t) = X(t)(1-X^2(t)) + \vas \sqrt{1+X(t)^2} dW(t).
\end{align*}
The multiplicative noise adds to the complexity of the already complex nonlinear dynamics of the original double-well process. The stationary density of the SDE is given by 
\begin{align}\label{eq:DW-v}
\pi_{st}(x) \propto \vas^{-2}(1+x^2)^{2\vas^{-2}-1}\exp\lf(-x^2/\vas^2\ri)
\end{align}
The stationary distribution is bimodal if $\vas <1$, but it becomes unimodal if $\vas\geq 1$ with sharper peak with increasing $\vas$. We consider two cases, $\vas=1$ and $\vas=0.5$. 

\np
{\bf Case: $\vas=1$:} We first consider (discrete) observations from \eqref{eq:DW-v} with true $\vas=1$, and use Algorithm \ref{algo:Gibbs-t} and Algorithm \ref{algo:Gibbs-HS} to estimate the drift function $b$ and the diffusion parameter $\vas$. For Algorithm \ref{algo:Gibbs-t}, we use the same $\tdst$-distribution as the last example. For Algorithm \ref{algo:Gibbs-HS}, classical HS prior was shrinking all the weights $\wtb_k$ to near $0$, and it was necessary to use heavier-tailed distribution on the local variance component $\loc_k$ (than $\Fdst(\nu_1=1,\nu_2=1, c=1)$-distribution) to counter the strong global shrinkage effect of $\glo$. We use $\Fdst(\nu_1=1,\nu_2=0.3, c=1)$-distribution on $\loc_k$ and the usual $\Fdst(\nu_1=1,\nu_2=1, c=1)$-distribution on $\glo$, that is, the following values of hyperparameters:  $\hyp_i = 0.5, \hyp^0 = \hpa= \hpa^0= 1/2, \hpb=\hpb^0=1$. As before, we use $\SC{IG}(1,2)$-prior on the diffusion-parameter $\vas^2$ in both the algorithms. The efficacy of the algorithms is demonstrated in Figure \ref{fig:dw-v-1}. The MSE and the Kolmogorov metric values for cases corresponding to no-shrinkage, $\tdst$-prior and the above HS-type prior are comparable and are again in the range 0.24-0.27 and about 0.07, respectively.
The values of the estimate, $\hat{\vas^2}$, given by Algorithm \ref{algo:Gibbs-t} and Algorithm \ref{algo:Gibbs-HS}  are  $0.998$ and $0.974$, respectively.
Again, the global-local setup of a HS-type prior (Algorithm \ref{algo:Gibbs-HS}) was able to produce significantly higher shrinkage while achieving comparable level of accuracy.

As before, we list in Table \ref{table:dw-v-1:1} and Table \ref{table:dw-v-1:2} the values of MSE and the Kolmogorov metric in two cases  -  (i)  fixed observation-time range $[0,T]$, but increasing $\Delta$, and (ii) fixed observation frequency $\Delta$ but increasing time range $[0,T]$.
As expected, better accuracy is obtained with more observations, with increasing time range $[0,T]$ of observations being more important than a fixed one with higher frequency of observations (that is smaller $\Delta$). This is natural as data over longer time range reveals more about the behavior of the underlying SDE.

\begin{table}[h!]
\centering
\begin{tabular}{|p{2.4cm}||p{1.2cm}|p{1.2cm}|p{1.2cm}|p{1.2cm}|p{1.2cm}|p{1.2cm}|} 
 \hline
 $T$ & 40 & 40 & 40 & 80 & 60 & 20 \\ 
 $\Delta$ & 0.025 & 0.05 & 0.1 & 0.05 & 0.05 & 0.05 \\ [0.5ex] 
 \hline\hline
 $\tdst$-prior & 0.3609 & 0.4512 & 0.5632 & 0.2702 & 0.3443 & 0.8265  \\ 
 HS-type-prior & 0.3838 & 0.3972 & 0.7868 & 0.2404 & 0.3319 & 0.8858
 \\[1ex] 
 \hline
\end{tabular}
\caption{MSE with different priors.}
\label{table:dw-v-1:1}
\end{table}

\begin{figure}[H]
\centering
\includegraphics[width=0.8\textwidth]{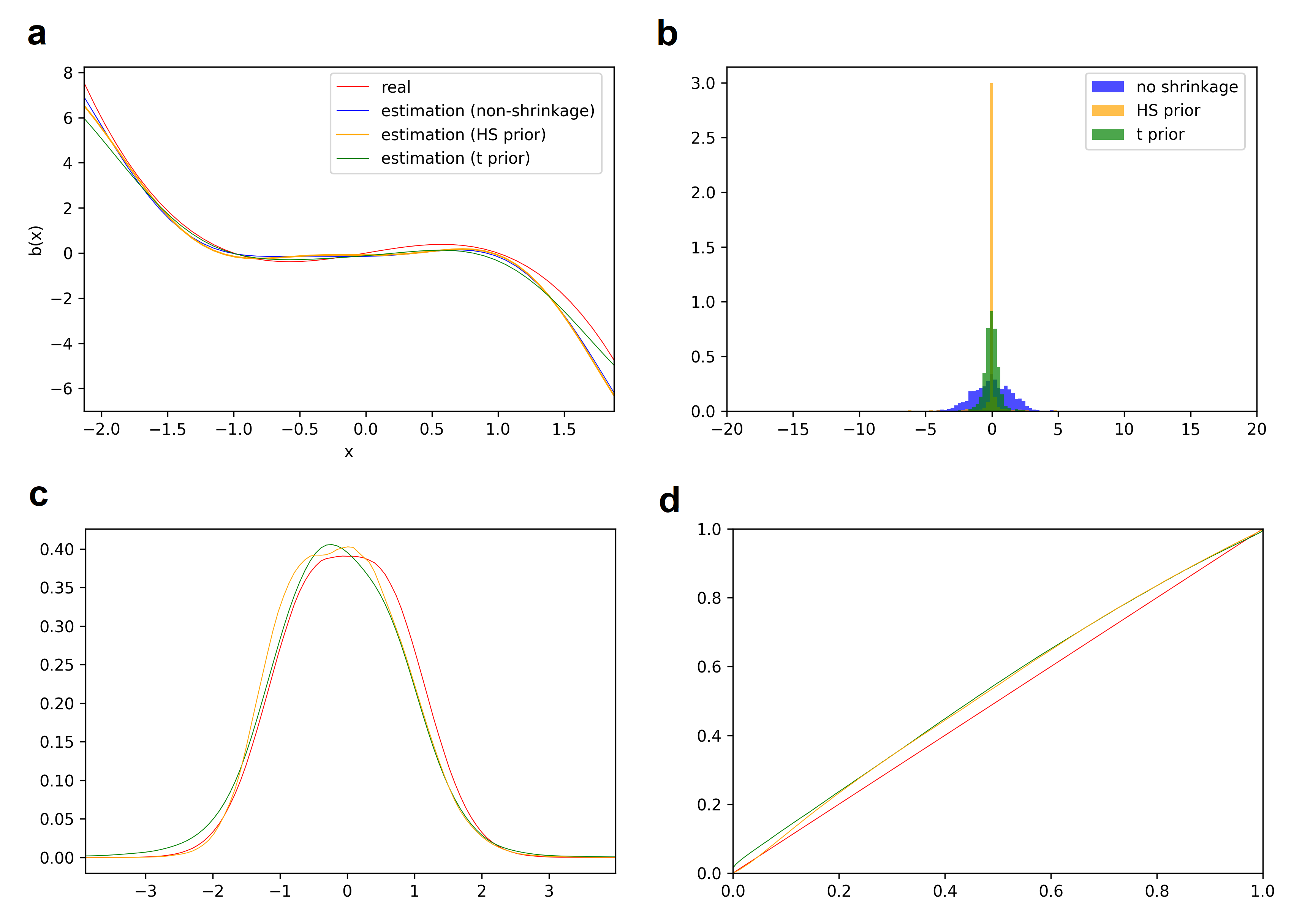}
\caption{Variant of double-well model ($\vas=1$). Descriptions of \textbf{a}, \textbf{b}, \textbf{c}, \textbf{d} are similar to that in Figure \ref{fig:dw}.
}
\label{fig:dw-v-1}
\end{figure}

\begin{table}[h!]
\centering
\begin{tabular}{|p{2.4cm}||p{1.2cm}|p{1.2cm}|p{1.2cm}|p{1.2cm}|p{1.2cm}|p{1.2cm}|} 
 \hline
 $T$ & 40 & 40 & 40 & 80 & 60 & 20 \\ 
 $\Delta$ & 0.025 & 0.05 & 0.1 & 0.05 & 0.05 & 0.05 \\ [0.5ex] 
 \hline\hline
 $\tdst$-prior & 0.107 & 0.1091 & 0.1291 & 0.071 & 0.1225 & 0.1826  \\ 
 HS-prior & 0.1187 & 0.107 & 0.1259 & 0.077 & 0.1268 & 0.1708 \\[1ex] 
 \hline
\end{tabular}
\caption{Kolmogorov metric between the CDFs with different prior.}
\label{table:dw-v-1:2}
\end{table}

\np
{\bf Case: $\vas=0.5$:} We also consider data points from \eqref{eq:DW-v} with $\vas = 0.5$ over the interval $[0,40]$ (with $\Delta=0.05$). As mentioned, the true stationary distribution in this case is distinctly bimodal. Bimodality and multiplicative noise make estimation of the drift function $b$ particularly a challenging task. Figure \ref{fig:dw-v-0.5} compares the estimated and the true $b$, and the corresponding stationary distributions. The hyperparameter values used in Algorithm \ref{algo:Gibbs-t} and Algorithm \ref{algo:Gibbs-HS} are the same as in the previous case. 

\begin{figure}[h]
\centering
\includegraphics[width=0.9\textwidth]{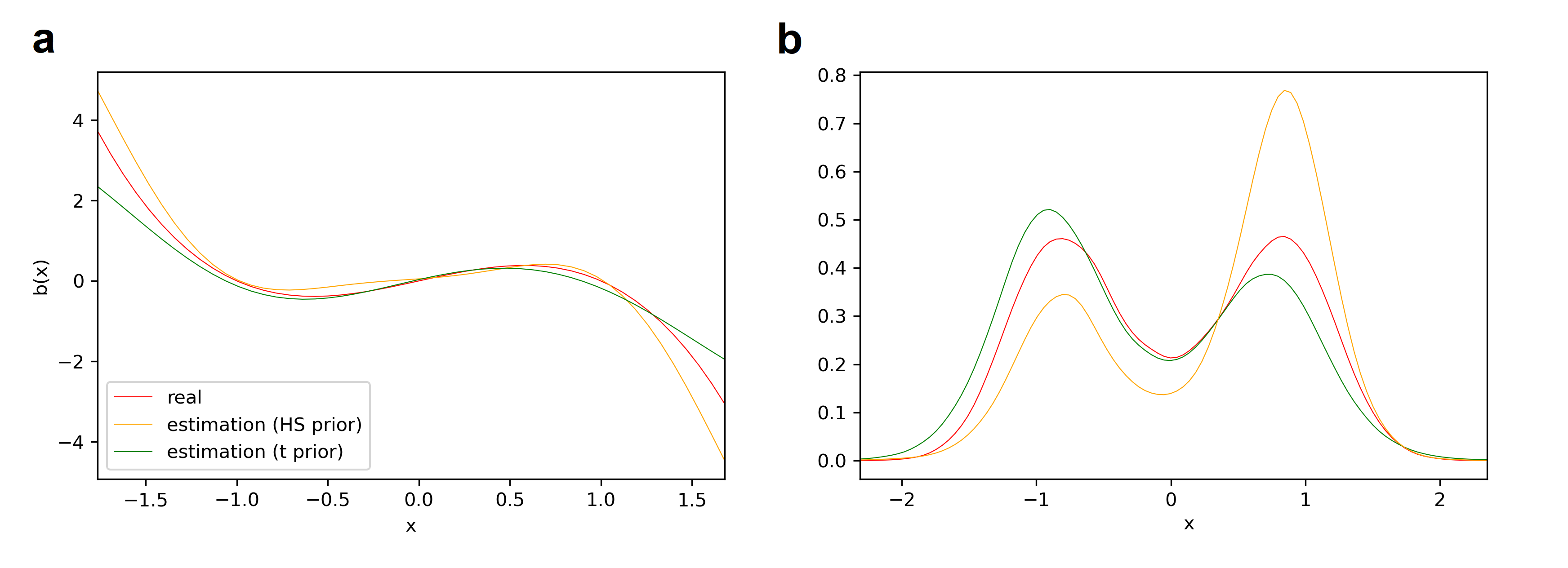}
\caption{Variant of double-well model. ($\vas=0.5$). \textbf{a}: comparison of estimated function $\hat b$ with true $b$. \textbf{b}: comparison of the stationary distributions of the SDE driven by estimated $\hat b$ and true $b$.}
\label{fig:dw-v-0.5}
\end{figure}
The estimated $\hat b$ for both the priors match closely with true $b$ (on a large part of the $x$-axis), and the estimator $\hat{\vas^2} = 0.247$ and $0.258$, which is almost same as the true $\vas^2 = 0.25$. But here Algorithm \ref{algo:Gibbs-t} with $\tdst$-prior on the weights gives a much more accurate result than Algorithm \ref{algo:Gibbs-HS} with the HS-type prior. This is clear from the plots of the different stationary distributions, where HS prior respectively underestimates and overestimates the modes at $-1$ and $1$. The MSE values for estimated $\hat b$ corresponding to $\tdst$ and HS priors are respectively 0.051 and 0.04, which are comparable. But the Kolmogorov metric between the CDFs of the true stationary distribution (c.f \eqref{eq:DW-v}) and the stationary distribution with $\hat b$ as the drift in the case of $\tdst$ and HS priors is respectively 0.05 and 0.17 showing the edge that the Algorithm  \ref{algo:Gibbs-t} had in this case.\\

\np
{\bf Model 3: Michaelis-Menten Kinetics}\\
The Michaelis-Menten is a well-known model in enzymatic kinetics describing the enzymatic substrate conversion process \cite{MiMe1913, Sri21}. The reaction system is given by
\begin{equation}
\label{MM_reactions}
E+ S \xrightleftharpoons[k_{2}]{k_{1}} ES, \quad ES \xrightleftharpoons[k_{m2}]{k_{m1}}  E+P. 
\end{equation}
The full state of the system at time $t$ is given by $X(t) = (X_{E}(t), X_{S}(t), X_{ES}(t), X_{P})$.
The system satisfies the conservation law: $X_{E}(t)+X_{ES}(t) = X_E(0) +X_{ES}(0)$. This gives a reduced $3$-dimensional state which will still be denoted by $X(t) = (X_{E}(t), X_{S}(t), X_{P}).$ The differential equation describing the dynamics is governed by the drift function 
\begin{align*}
b(x) = (-k_1 x_E x_S - k_{m2} x_E x_P + (k_{m1} + k_2) x_{ES}, -k_1 x_E x_S + k_{m1} x_{ES},  k_2 x_{ES} - k_{m2} x_E x_P  )
\end{align*}
Given a set of discrete observations from a stochastic version of this differential equation driven by additive Brownian noise $\vas_{3\times 3} B$, with $\vas = 0.1 I$ over time-range $[0,40]$  generated by taking $\Delta=0.04$ and the conservation constant, $X_E(0)+X_{ES}(0)=2$, we use Algorithm \ref{algo:Gibbs-t} and Algorithm \ref{algo:Gibbs-HS} to estimate the entire drift function $b$ and the (constant) diffusion matrix $\vas$. For Algorithm \ref{algo:Gibbs-t}, we use the hyperparameter values, $\nu=5,$ $U=8I$, and $\SC{IW}_3(1+\text{dim},V=2I_{3\times 3})$-prior (where, dimension, dim$=3$) on $\vas\vas^T$. For Algorithm \ref{algo:Gibbs-t} we use the (multidimensional version of) classical HS prior, and the same inverse-wishart prior on $\vas\vas^T$. The MSE values in both cases came out to be about 0.004 (specifically, 0.00414 for HS and 0.00431 for $\tdst$). Figure \ref{fig:mm} -  \textbf{a}, \textbf{b} and \textbf{c}, respectively, plots the first, second and third component of both the estimator $\hat b$ and the true $b$ when HS-prior is used with $z$-coordinate fixed at $1.073$.
\begin{figure}[H]
\centering
\includegraphics[width=0.9\textwidth]{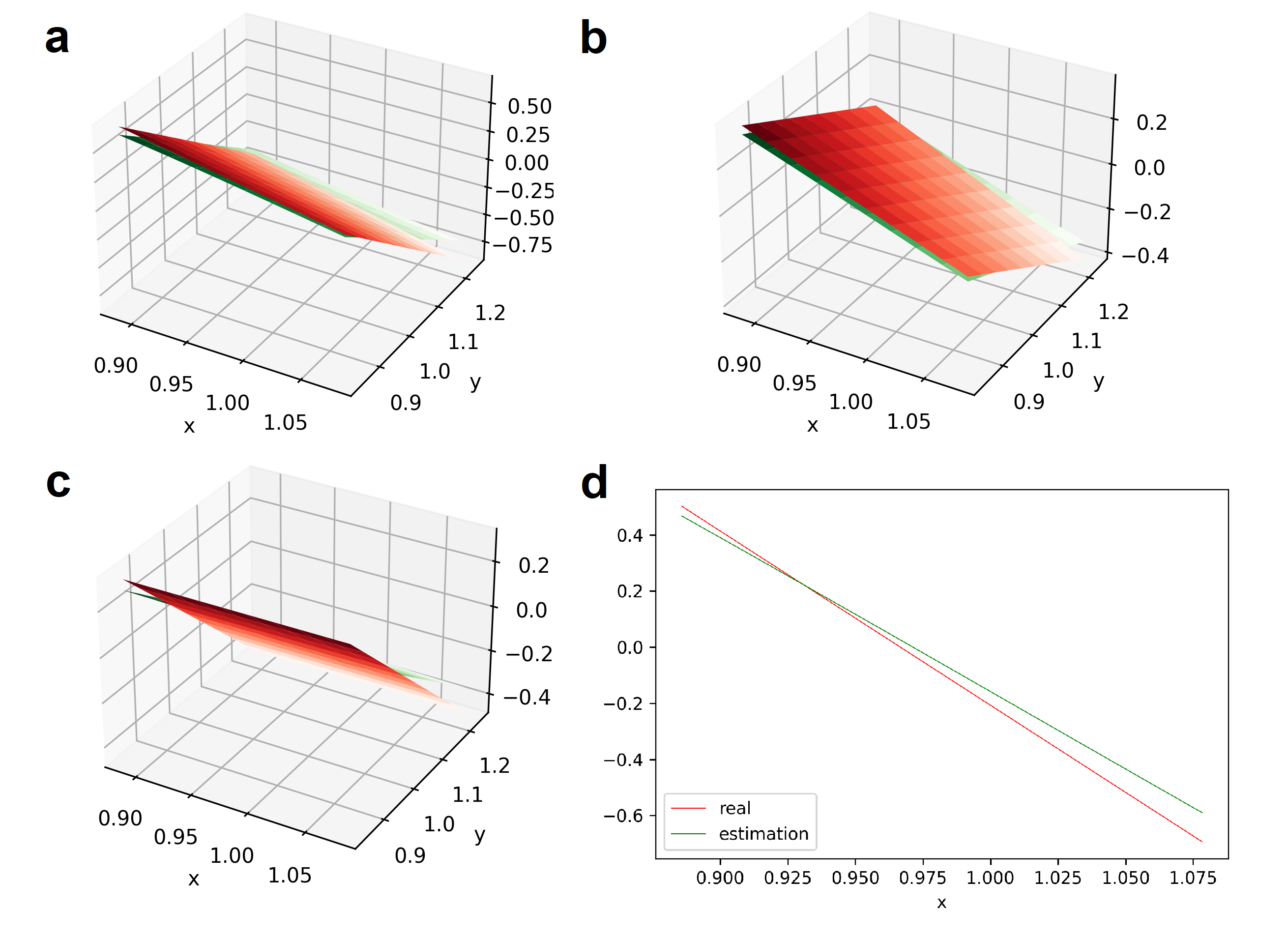}
\caption{Michaelis-Menten Kinetics model with HS-prior. \textbf{a}, \textbf{b}, \textbf{c} show the plots of first, second and third component of the functions $\hat b$ and $b$ with $z$ fixed at $1.073$. \textbf{d} shows a two-dimensional slice of \textbf{a} at $y=1.060$.}
\label{fig:mm}
\end{figure}
The estimated diffusion matrix (via Algorithm \ref{algo:Gibbs-HS}) is given by 
\begin{align*}
\hat{\vas\vas^T} = 
\begin{bmatrix}
    0.01210109 & 0.0001914 & -0.00020334 \\
    0.0001914 &  0.01170018 & 0.00014556 \\
    -0.00020334 & 0.00014556 & 0.01122171 \\
\end{bmatrix}
\end{align*}
which is very close to the true $\vas\vas^T=0.01I$. The corresponding numbers for Algorithm \ref{algo:Gibbs-t} are very similar.

\section{Discussion} \label{sec:dis}

The paper presents a novel theoretical and computational paradigm for  stochastic dynamic models which, on account of its generalizability,  can potentially find its way to several interesting applications. We study two areas --- (a) a class of infinite-dimensional optimization problems, which is broader than what the classical Representer Theorem covers, (b) Bayesian approach to nonparametric inference of stochastic dynamical systems. To our knowledge, this is the first instance of the merging of Bayesian methods, RKHS theory and stochastic differential equations into a single unified platform.
The use of the resulting algorithms on data from well-known SDEs amply demonstrates their ability to learn the true drift functions  to a high degree of accuracy. Specifically, their reliable prediction of long term dynamics beyond the range of data points is a strong testament to this fact. The accuracy measures obtained under the M-M kinetics model lend strong  credence to the relevance of this approach for multivariate settings.

The hierarchical structure of the Bayesian framework  makes the resulting inference scheme computationally scalable while  opening the door to several model extensions. For instance, a semi-parametric model variant  could  be easily implemented  in instances where the stochastic dynamics is known only partially. It is also of interest to study the effectiveness of other types of shrinkage priors in this context. These  extensions would also broadly benefit from the convenience of Gibbs sampling schemes similar to the ones showcased in this article.  The  ‘divide and conquer’ approach  intrinsically encoded in such schemes would  typically allow for  multiple  computational  conveniences, like  parallel computation as and when required.

Several ongoing works are focusing on more general models including sparse and noisy datasets, dynamical systems with jumps and  multiscale stochastic systems,  each of which has its own unique challenges. For example, for SDE models with noisy data the expansion in Theorem \ref{th-rep} does not directly hold as the actual trajectory of the underlying SDE is never observed. The generality of the optimization results in the first part of the paper will play a key role in these cases.

\setcounter{section}{0}
\setcounter{theorem}{0}
\setcounter{equation}{0}
\renewcommand{\theequation}{\thesection.\arabic{equation}}

\appendix
\section{Appendix}

 {\bf[Proof of Lemma \ref{lem:pd-op}]}
 \np
 (i) $\RT$ (ii):  Suppose that $\{Qh_n\} \subset \ran(Q)$ such that $Qh_n \rt g$ as $n\rt \infty$. We need to show that $g = Qh$ for some $h \in \Hsp.$ Notice that in particular $\{Qh_n\}$ is Cauchy. Since $Q$ is uniformly p.d, there exists $\l >0$ such that $\<Qh,h\>  \geq \l \|h\|^2$. This implies $\{h_n\}$ is also a Cauchy sequence, and hence by completeness of $\Hsp$ there exists $h$ such that $h_n \rt h$. By continuity of $Q$, we then have $Qh_n \rt Qh$, and therefore, $Qh=g.$\\

\np
(ii) $\RT$ (iii): Suppose $\ran(Q)^\perp \neq \{0\}$. Let $0\neq y \in \ran(Q)^\perp$. Now $Qy \in \ran(Q)$; hence $\<Qy,y\> =0$. But since $Q$ is p.d this means that $y=0$; in other words, $\ran(Q)^\perp = \{0\}$. Since $\ran(Q)$ is closed by the hypothesis, we get from $\Hsp = \ran(Q)\oplus \ran(Q)^\perp$ that $\ran(Q) = \Hsp.$\\

\np
(iii) $\RT$ (i): Observe that since $Q$ is self-adjoint and p.d, $\<h',h\>_Q \stackrel{def}= \<Qh',h\>$ defines a valid inner product. Furthermore, since $Q$ is surjective, $Q^{-1}$ is a bounded linear operator, that is, $Q^{-1} \in L(\Hsp, \Hsp)$.
Now by Cauchy-Schwartz inequality, 
$|\<h',h\>_Q| \leq \|h'\|_Q \|h\|_Q.$
Taking $h' = Q^{-1}h$, we get $\|h\| \leq \|Q^{-1}\|^{1/2} \|h\|_Q$ which establishes (i).
\qed 

\vs{.2in} \np
Recall the notations described before Algorithm \ref{algo:Gibbs-HS}

\begin{lemma} \label{lem:cond-0}
Suppose that the joint distribution of 
 $\BX_{t_1:t_m} =(X(t_1), X(t_2), \hdots, X(t_m)))$ given the parameters $\bm{\wtb}$ and $\vas\vas^T$ is described by the transition probabilities \eqref{eq:em-dens}. Assume that 
 \begin{itemize}
 \item $\wtb_i| \loc_i, \glo \sim \No_d(\cdot|0,\loc_i\glo I),$
 \item $\loc_i|\theta_i \sim \SC{IG}(\cdot|\hyp_i,\theta_i)$ for $i=1,2,\hdots,m$, $\glo | \theta^0 \sim \SC{IG}(\cdot|\hyp^0,\theta^0)$;
 \item $\theta^0,\theta_i, i=1,2,,\hdots,m$ are independent, and for each $i=1,2,\hdots, m$, $\theta_i  \sim \SC{G}(\cdot|\hpa,\hpb)$ and $\theta^0  \sim \SC{G}(\cdot|\hpa^0,\hpb^0)$
 \item $\vas\vas^T$: $\vas\vas^T  \sim \SC{IW}_d(\cdot|n,V)$.
  \end{itemize}
Then
	\begin{enumerate}[(i)]
		\item $\bm{\wtb}|\salg_{-\bm{\wtb}} \ \sim\ N(\cdot|\bmn, \bcm)$ where 
		\begin{align}
		\label{eq:b-post}
		\begin{aligned}
		\bcm^{-1} =&\ \Delta\bm{\SC{K}}_0^T\sdg\bm{\SC{K}}_0+ \eta^{-1}, \quad \bmn= \bm{C}\bm{\SC{K}}^T_0\sdg \bm{\vartheta}\\
		\sdg_{dm\times dm} =&\ \mbox{diag}\lf((\s\s^T(X(t_1)))^{-1}, (\s\s^T(X(t_2)))^{-1} , \hdots, (\s\s^T(X(t_m)))^{-1}\ri)\\
		\eta_{dm\times dm} =&\ \mbox{diag}\lf(\loc_1\glo, \loc_2\glo,\hdots, \loc_m\glo \ri)\ot I_d\\
		\bm{\vartheta}_{dm\times 1} =&\ \ve_{d\times m}\lf(X(t_1) -x_0, X(t_2) - X(t_1),\hdots, X(t_m) - X(t_{m-1})\ri) 
		\end{aligned}
		\end{align}
		
		\item $\vas\vas^T| \salg_{-\vas\vas^T}  \ \sim\ \SC{IW}_d(n+m,V_{post})$, where
		\begin{align}\label{eq:sigma-post}
		V_{post} =\Delta^{-1}\sum_{k=1}^m \lf(\s_0(X(t_{j-1}))\ri)^{-1}(\vartheta_j - b(X(t_j)\Delta)(\vartheta_j  - b(X(t_j)\Delta)^T\lf(\s^T_0(X(t_{j-1}))\ri)^{-1}+V
		\end{align}
		
		\item Conditioned on $\salg_{-\{\loc_i\}}$,  $\loc_k,\ k=1,2,\hdots,m$ are independent, and  
		\begin{align*}
		\loc_k| \salg_{-\{\loc_i\}}\ \sim\ \SC{IG}\lf(\cdot \Big|(d+2\hyp_k)/2,\  \f{1}{2}\wtb^T_k\wtb_k/\glo+\theta_k\ri)
     	\end{align*}
     	
     	\item $
     	\glo|  \salg_{-\glo} \ \sim\ \SC{IG}\lf(\cdot \Big| (md+2\hyp^0)/2,\  \theta^0+\f{1}{2}\sum_{k=1}^m \wtb_k^T\wtb_k/\loc_k\ri)
     	$
		
		\item Conditioned on $\salg_{-\theta^0,\{\theta_i\}}$, $\theta^0, \theta_k,\ k=1,2,\hdots,m$ are independent
		\begin{align*}
		\theta_k| \salg_{-\theta^0,\{\theta_i\}}\ \sim\ \SC{G}\lf(\cdot|\hyp_k+\hpa,\hpb+1/\loc_k\ri), \quad \theta| \SC{G}_{-\theta^0,\{\theta_i\}}\ \sim\ \SC{G}\lf(\cdot|\hyp^0+\hpa^0,\hpb^0+1/\glo\ri)
		\end{align*}
		
	\end{enumerate}
	
\end{lemma}

\begin{proof}
	By a slight abuse of notation, we use $f$ as a generic symbol for various conditional densities below.
	Notice that 
	$$f\lf(\bm{\wtb}|\salg_{-\bm{\wtb}}\ri) \propto \exp\lf\{-\f{1}{2}\lf(\SC{E}_0 +\sum_{k=1}^m\wtb_k^T\wtb_k/(\loc_k\glo)\ri)\ri\},$$
	where 
	\begin{align*}
	\SC{E}_0 =&\ \Delta^{-1}\sum_{k=0}^{m-1}\lf[(\vartheta_{k+1} - \Delta b(X(t_k)))^T\lf( \s\s^T(X(t_k))\ri)^{-1}(\vartheta_{k+1} - \Delta b(X(t_k)))\ri]\\
	=&\ \sum_{k=0}^{m-1} \Delta^{-1}\vartheta_{k+1}^T \lf( \s\s^T(X(t_k))\ri)^{-1} \vartheta_{k+1} + \Delta b^T(X(t_k) \lf( \s\s^T(X(t_k))\ri)^{-1} b(X(t_k))\\
	& \hs{.2cm} -2 \vartheta_{k+1}^T\lf( \s\s^T(X(t_k))\ri)^{-1}b(X(t_k)).
	\end{align*}
	Here $\vartheta_{k+1} = X(t_{k+1}) - X(t_{k}),$ and  recall that $b(x) = \sum_{k=1}^m\knl_0 (x, X(t_k)).$
	Now
	$$b(X(t_k)) = \sum_{j=1}^m \knl_0(X(t_k), X(t_j))\wtb_j = \bm{\SC{K}}_0(X(t_k),*) \bm{\wtb},$$
	and hence
	\begin{align*}
	\sum_{k=0}^{m-1}\vartheta_{k+1}^T\lf( \s\s^T(X(t_k))\ri)^{-1}b(X(t_k)) = &\ \bm{\vartheta}^T \sdg \bm{\SC{K}_0}\bm{\wtb}\\
	\sum_{k=0}^{m-1}b^T(X(t_k) \lf( \s\s^T(X(t_k))\ri)^{-1} b(X(t_k))=& \bm{\wtb}^T \bm{\SC{K}_0}^T\sdg \bm{\SC{K}_0}\bm{\wtb}.
	\end{align*}
	Since $\sum_{k=1}^m\wtb_k^T\wtb_k/(\loc_k\glo) =  \bm{\wtb}^T \bm{\eta}^{-1} \bm{\wtb}$, it follows that 
	\begin{align*}
	f\lf(\bm{\wtb}|\salg_{-\bm{\wtb}}\ri) = N(\bm{\wtb} | \bmn, \bcm)
	\end{align*}
	where
	\begin{align*}
	\bcm^{-1} = \Delta  \bm{\SC{K}_0}^T\sdg \bm{\SC{K}_0}+\bm{\eta}^{-1}, \quad \bmn = \bcm\bm{\SC{K}_0}\sdg \bm{\vartheta}
	\end{align*}
	with $\sdg$ as in \eqref{eq:b-post}.
	Next note that
	\begin{align*}
	f\lf(\vas\vas^T|\salg_{-\vas\vas^T}\ri) \propto&\ \det\lf((\vas\vas^T)\ri)^{-m/2}\exp\lf\{-\f{1}{2}\SC{E}_0\ri\}
	\det\lf((\vas\vas^T)\ri)^{\f{-(n+d+1)}{2}}\exp\lf\{-\f{1}{2}\tr\lf(V(\vas\vas^T)^{-1}\ri)\ri\}\\
	=& \det\lf((\vas\vas^T)\ri)^{\f{-(n+m+d+1)}{2}}\exp\Big\{-\f{1}{2}\tr\Big[\Big(\Delta^{-1}\sum_{k=1}^m\s^{-1}_0(X(t_k))(\vartheta_{k+1} - \Delta b(X(t_k))) \\
	& \ \times  (\vartheta_{k+1} - \Delta b(X(t_k)))^T\lf(\s^{T}_0(X(t_k))\ri)^{-1}+V\Big)(\vas\vas^T)^{-1}\Big]\Big\}
	\end{align*}
	which proves the assertion.
	Next note that
	\begin{align*}
	f\lf(\loc_1,\loc_2,\hdots,\loc_m|\salg_{-\{\loc_i\}}\ri) \propto&\ (\glo)^{-md/2}\prod_{k=1}^m (\loc_k)^{-d/2}\exp\lf\{-\f{1}{2}(\loc_k\glo)^{-1}\wtb_k^T  \wtb_k\}\ri\} \\
	& \hs{.2cm} \times \prod_{k=1}^m (\loc_k)^{-(\hyp_k+1)} \exp\lf\{-\f{\theta_k}{\loc_k}\ri\}\\
	\propto &\ \prod_{k=1}^m  (\loc_k)^{-(d+2\hyp_k)/2-1}\exp\lf\{-\lf(\f{1}{2}\wtb^T_k\wtb_k/\glo+\theta_k\ri)/\loc_k\ri\}.
	\end{align*}
	which proves the assertion.
	Similarly, 
	\begin{align*}
	f\lf(\glo|\salg_{-\glo}\ri) =&\ (\glo)^{-md/2}\prod_{k=1}^m (\loc_k)^{-d/2}\exp\lf\{-\f{1}{2}(\loc_k\glo)^{-1}\wtb_k^T  \wtb_k\}\ri\} \times 
 (\glo)^{-(\hyp^0+1)} \exp\lf\{-\f{\theta^0}{\glo}\ri\}\\
	\propto &\  (\glo)^{-(md+2\hyp^0)/2-1}\exp\lf\{-\lf(\theta^0+\f{1}{2}\sum_{k=1}^m \wtb_k^T\wtb_k/\loc_k\ri)\Big/\glo\ri\}.
	\end{align*}
	and (iv) follows.
Finally notice that 
	\begin{align*}
f\lf(\theta, \theta_1,\theta_2,\hdots,\theta_m|\salg_{-\theta,\{\theta_i\}}\ri) \propto&\  \prod_{k=1}^m \theta_k^{\hyp_k}(\loc_k)^{-(\hyp_k+1)} \exp\lf\{-\f{\theta_k}{\loc_k}\ri\}  \times \theta^{\hyp^0} (\glo)^{-(\hyp^0+1)} \exp\lf\{-\f{\theta^0}{\glo}\ri\} \\
& \hs{.2cm} \times \prod_{k=1}^m \theta_k^{\hpa-1}\exp\lf\{-\hpb\theta_k\ri\} \times (\theta^0)^{\hpa^0-1}\exp\lf\{-\hpb^0\theta^0\ri\}  \\
\propto &\ \prod_{k=1}^m \theta_k^{\hyp_k+\hpa-1} \exp\lf\{-(\hpb+1/\loc_k)\theta_k\ri\} \times  (\theta^0)^{\hyp^0+\hpa^0-1} \exp\lf\{-(\hpb^0+1/\glo)\theta\ri\} 
\end{align*}	
which proves (v).

\end{proof}	

\bibliographystyle{plainnat}
\bibliography{Ref-ML}

\end{document}